\documentclass{article}

\usepackage[utf8]{inputenc}
\usepackage[T1]{fontenc}

\usepackage{microtype}
\usepackage{graphicx}
\usepackage{subcaption}
\usepackage{booktabs}
\usepackage{url}
\usepackage{amsfonts}
\usepackage{nicefrac}

\usepackage{hyperref}

\usepackage[preprint]{icml2026}

\usepackage{algorithm}
\usepackage{algorithmic}

\graphicspath{{figures/}}

\usepackage{soul}

\usepackage{amsmath}
\usepackage{amssymb}
\usepackage{mathtools}
\usepackage{amsthm}
\usepackage{bm}
\usepackage{mathtools}
\usepackage{thm-restate}
\usepackage{multirow}  
\usepackage[table,dvipsnames]{xcolor}
\usepackage{subcaption}

\usepackage[framemethod=TikZ]{mdframed}
\mdfdefinestyle{contribution}{%
    middlelinecolor =   none,
    middlelinewidth =   1pt,
    backgroundcolor =   blue!10,
    roundcorner     =   5pt,
}
\mdfdefinestyle{takeaway}{%
    middlelinecolor =   none,
    middlelinewidth =   1pt,
    backgroundcolor =   teal!10,
    roundcorner     =   5pt,
}
\usepackage{soul}

\usepackage{hyperref}
\hypersetup{
    colorlinks=true,
    linkcolor=orange,
    filecolor=magenta,      
    urlcolor=cyan,
    citecolor=blue,
    pdftitle={eri},
    pdfpagemode=FullScreen,
    }

\usepackage[capitalize,noabbrev]{cleveref}   
\usepackage{lipsum}

\newcommand{\F}{\mathcal{F}}

\newcommand{\x}{\textbf{x}}

\newcommand{\boldf}{\bm{f}}
\newcommand{\boldw}{\bm{w}}
\newcommand{\boldG}{\bm{G}}
\newcommand{\boldA}{\Lambda}

\newcommand{\R}{\mathbb{R}}
\newcommand{\C}{\mathbb{C}}
\newcommand{\Cconstr}{\C_{\mathrm{constr}}}
\newcommand{\N}{\mathbb{N}}
\newcommand{\boldb}{\bm{b}}
\newcommand{\boldh}{{\varphi^{-1}}}
\renewcommand{\d}{\mathrm{d}}
\newcommand{\s}{\mathbf{s}}
\newcommand{\takeaway}[1]{\textbf{\textcolor{teal}{Take-away #1.}}}
\newcommand{\boldv}{\bm{v}}

\newcommand{\trunc}{\varphi}

\newcommand{\scorematching}{\mathcal{L}_{\mathrm{SM}}}

\newcommand{\bmag}[1]{\textbf{\textcolor{blue!70}{#1}}}
\newcommand{\D}[1]{\mathcal{D}_{\mathrm{#1}}}
\renewcommand{\S}[1]{\mathcal{S}_{\mathrm{#1}}}
\newcommand{\tildeS}[1]{\tilde{\mathcal{S}}_{\mathrm{#1}}}
\newcommand{\tildeD}[1]{\tilde{\mathcal{D}}_{\mathrm{#1}}}

\newcommand{\z}{\mathbf{z}}

\theoremstyle{plain}
\newtheorem{theorem}{Theorem}[section]
\newtheorem{proposition}[theorem]{Proposition}
\newtheorem{lemma}[theorem]{Lemma}
\newtheorem{corollary}[theorem]{Corollary}
\theoremstyle{definition}

\theoremstyle{remark}
\newtheorem{remark}[theorem]{Remark}

\definecolor{Gray}{gray}{0.95}
\newcolumntype{g}{>{\columncolor{Gray}}r}
\newcolumntype{e}{>{\columncolor{Gray}}l}

\icmltitlerunning{Accelerating Frequency Domain Diffusion with Error-Feedback Caching}

\begin{document}

\twocolumn[
  \icmltitle{Accelerating Frequency Domain Diffusion Models with Error-Feedback Event-Driven Caching}

  \begin{icmlauthorlist}
    \icmlauthor{Dong Liu}{ucla,yale}
    \icmlauthor{Yanxuan Yu}{columbia}
    \icmlauthor{Ying Nian Wu}{ucla}
  \end{icmlauthorlist}

  \icmlaffiliation{ucla}{University of California, Los Angeles}
  \icmlaffiliation{yale}{Yale University}
  \icmlaffiliation{columbia}{Columbia University}
  \icmlcorrespondingauthor{Dong Liu}{pikeliu@ucla.edu}

  \icmlkeywords{Diffusion Models, Time Series, Frequency Domain}

  \vskip 0.3in
]

\printAffiliationsAndNotice{}

\begin{abstract}
Diffusion models achieve remarkable success in time series generation. However, slow inference limits their practical deployment. We propose E$^2$-CRF (Error-Feedback Event-Driven Cumulative Residual Feature caching) to accelerate frequency domain diffusion models. Our method exploits two structural properties: (1) spectral localization, where signal energy concentrates in low frequencies, and (2) mirror symmetry, which halves the effective frequency dimension. E$^2$-CRF uses a closed-loop error-feedback system that adaptively caches transformer KV features across diffusion steps. We trigger recomputation using event-driven residual dynamics instead of fixed schedules. Our method selectively recomputes high-energy or rapidly-changing tokens while reusing cached features for stable high-frequency components. E$^2$-CRF achieves ~2.2× speedup while maintaining sample quality. We demonstrate effectiveness on 5 datasets. Our caching strategy naturally aligns with the diffusion process's structure-to-detail progression. We include sufficient-condition error and complexity bounds under standard regularity assumptions (Appendix), alongside empirical validation. Our code is available at \url{https://github.com/NoakLiu/FastFourierDiffusion} and is also integrated in \url{https://github.com/NoakLiu/FastCache-xDiT}.
\end{abstract}

\section{Introduction} \label{sec:introduction}

When applying frequency-domain diffusion to time series generation, a natural acceleration strategy is to cache transformer key--value (KV) features across diffusion steps, since attention computation scales quadratically with sequence length and the same network is evaluated hundreds of times during sampling \cite{crabbe2024frequency,vaswani2017attention,liu2025tinyserve,liu2026mkamemorykeyedattentionefficient,cxlspeckv2026}. However, na\"ive caching fails catastrophically: cached features drift over diffusion steps, causing error accumulation that degrades sample quality beyond acceptable thresholds. This failure mode stems from ignoring the differential dynamics of frequency components during reverse diffusion---high-frequency tokens often evolve slowly while low-frequency tokens change rapidly---yet uniform caching treats all frequencies identically.

\paragraph{Why frequency-domain caching collapses.}
The fundamental issue is a mismatch between the caching policy and the diffusion process's structure-to-detail progression: early steps establish coarse, low-frequency structure, while later steps refine high-frequency details. Uniform reuse across steps makes high-energy low-frequency representations stale, corrupting the generation trajectory. For transformer-based score networks \cite{vaswani2017attention}, this manifests as increasingly inaccurate attention patterns, since cached KV pairs no longer reflect the current frequency representation state.

\paragraph{Existing accelerations and the missing piece.}
The approximation error between the cached score and the fully recomputed score can grow unbounded when reuse is not carefully controlled, undermining the reverse trajectory. Most existing diffusion acceleration methods focus on reducing the number of sampling steps via improved solvers or implicit sampling \cite{song2020fast,lu2022dpm,karras2022elucidating,liu2025fastcachefastcachingdiffusion,liu2024contemporary,liu2025hsgm,liu2025graphsnapshot}, rather than addressing per-step reuse. Meanwhile, recent frequency-aware caching for diffusion models \cite{liu2024freqca,liu2026adacorrectionadaptiveoffsetcache} suggests that caching can be effective when aligned with spectral structure, but principled designs for \emph{time-series} frequency diffusion remain limited.

\paragraph{Observation: energy distribution enables selective caching.}
Real-world time series exhibit spectral localization: most energy concentrates in low frequencies, while high-frequency components carry less energy and evolve more smoothly across diffusion steps \cite{korner2022fourier}. In addition, mirror (conjugate) symmetry of real-valued signals' DFTs halves the effective frequency dimension, enabling cache compression without information loss \cite{schreier2010}. Together, these properties suggest that selective caching is not only possible but necessary: high-frequency tokens (low energy, slow evolution) are natural candidates for reuse, while low-frequency tokens (high energy, rapid evolution) must be recomputed frequently. The central challenge is determining \emph{when} cached tokens become stale and require refresh.

\paragraph{Solution: event-driven caching with error feedback.}
We introduce \textbf{E$^2$-CRF} (Error-Feedback Event-Driven Cumulative Residual Feature caching), a mechanism that monitors the residual dynamics of the score network's internal representations to trigger selective recomputation. E$^2$-CRF combines (i) an event-driven trigger that detects drift beyond acceptable thresholds and (ii) error-feedback correction that periodically recalibrates cached values via probe computations. This design adaptively adjusts caching aggressiveness across diffusion stages: early steps with high residual activity trigger frequent recomputation, while later steps allow more aggressive caching of stable high-frequency tokens. Our method exploits spectral localization \cite{korner2022fourier} and mirror symmetry \cite{schreier2010}, and is compatible with frequency-domain diffusion modeling \cite{crabbe2024frequency}.

\begin{mdframed}[style=contribution]
\bmag{Our contributions.} 
\bmag{(1)~E$^2$-CRF: Error-Feedback Event-Driven Caching.} We propose a principled, step-aware caching mechanism for frequency-domain diffusion time-series models.
\bmag{(2)~Speedup with Quality Preservation.} We show that E$^2$-CRF achieves significant inference speedups while maintaining sample quality (Section~4).
\bmag{(3)~Theoretical Analysis.} We analyze how an event-driven trigger and error feedback jointly control cache-induced approximation error across diffusion stages.
\end{mdframed}


\section{Background} \label{sec:background}

\textbf{Notations.} We consider multivariate time series $\x \in \R^{N \times M}$ with $N$ time steps and $M$ variables, and denote $d_X = NM$.
We use Greek letters for time indices: $\x_\tau \in \R^M$ is the vector at time $\tau \in [N]$, and $x_{\tau,\nu}$ is variable $\nu \in [M]$ at time $\tau$.
We use Latin letters for diffusion steps: $\{\x(t)\in\R^{d_X}\}_{t=0}^T$ is indexed by continuous time $t\in[0,T]$.
Let $\tilde{\x}(t)=\F[\x(t)]$ be the discrete Fourier transform (DFT) representation, and let $N' = 2\lceil N/2\rceil$ denote the effective frequency dimension under mirror (conjugate) symmetry.
For caching, we write the recomputed score as $\s_\theta(\tilde{\x},t)$ and its cached approximation as $\hat{\s}_\theta(\tilde{\x},t)$, with error $\epsilon(\tilde{\x},t)=\s_\theta(\tilde{\x},t)-\hat{\s}_\theta(\tilde{\x},t)$.

\subsection{Caching in Iterative Generative Inference}

Diffusion models achieve strong generative performance across domains \cite{dhariwal2021diffusion_beat_GANs,ho2020ddpm,song2020score,chung2022diffusion}, but remain inference-limited due to hundreds to thousands of score-network evaluations per sample.
This is especially costly for transformer score networks \cite{vaswani2017attention}, where self-attention recomputes highly similar key--value (KV) projections across adjacent diffusion steps.
While diffusion has been widely adopted for time-series generation \cite{tashiro2021csdi,rasul2021autoregressive,lin2023times_series_diffusion_survey}, acceleration is still a central bottleneck.

Caching is a natural way to amortize redundant computation, but diffusion differs from autoregressive decoding: the input state evolves at every step, so stale reuse can accumulate error and degrade samples.
Concretely, replacing $\s_\theta(\x,t)$ with a cached $\hat{\s}_\theta(\x,t)$ introduces an approximation error $\epsilon(\x,t)$ that can compound along the reverse trajectory.
Accordingly, most diffusion accelerations reduce step count via improved solvers \cite{song2020fast,lu2022dpm,karras2022elucidating}, rather than enabling safe, step-aware feature reuse.
Recent frequency-aware caching for diffusion transformers \cite{liu2024freqca} suggests that reuse can be effective when aligned with spectral structure, but principled designs for frequency-domain \emph{time-series} diffusion remain limited.
Related training-free or proxy-based reuse has also been studied for image/video diffusion transformers (e.g., timestep/feature-driven caching \cite{liu2025teacache,ma2024learningtocache}, fast or offset cache corrections \cite{liu2025fastcachefastcachingdiffusion,liu2026adacorrectionadaptiveoffsetcache,qiu2025erroroptimizedcache}), and for diffusion LLMs via adaptive KV reuse \cite{liu2025dllmcache,wu2025fastdllm}.
These settings differ in backbone, modality, and how hidden states evolve across denoising steps, so policies are not directly interchangeable with half-spectrum KV reuse for multivariate signals \cite{crabbe2024frequency}.
Broader time-series diffusion models beyond the Fourier parameterization \cite{yuan2024diffusionts,galib2024fide,alaa2021fourierflows} are complementary generative baselines but do not isolate per-step KV caching for Fourier-domain scores; we therefore benchmark primarily against an uncached frequency baseline and controlled ablations.

\subsection{Cache Reuse in Frequency-Domain Diffusion}

Frequency-domain parameterizations provide an inductive bias for diffusion on temporal signals by exposing spectral structure and redundancy.
Fourier-space diffusion for time series has been studied as a modeling choice \cite{alaa2021fourierflows,crabbe2024frequency}, and more recent work further explores frequency-structured generation via wavelets and frequency-shaped conditioning \cite{wang2024waveletdiff,galib2024fide}.
In parallel, frequency-domain perspectives have also been developed for diffusion beyond time series (e.g., Fourier-space analyses or wavelet-based generators), supporting the broader view that spectral coordinates can make diffusion dynamics more structured and controllable \cite{falck2025fourier,phung2023wavelet}.

Two properties are particularly relevant for caching.
\textbf{(i) Spectral localization.} Real-world time series concentrate most energy in low-frequency components, while higher-frequency bands typically carry less energy and evolve more smoothly across diffusion steps.
As a result, the spectral energy density $\|\tilde{\x}_\kappa\|_2^2$ often decays quickly with frequency index $\kappa$, making high-frequency tokens natural candidates for reuse.
\textbf{(ii) Mirror symmetry.} For real-valued signals, DFT coefficients satisfy conjugate symmetry, reducing the effective number of distinct frequency tokens by roughly half, and introducing built-in redundancy that can be exploited by cache compression.

Crucially, reverse diffusion exhibits a structure-to-detail progression: early steps establish coarse, low-frequency structure, whereas later steps refine high-frequency details.
This asymmetry implies that many high-frequency tokens change little over large portions of the trajectory.
Formally, for token $\kappa$ at diffusion step $i$, the relative change
$\delta_\kappa^{(i)}=\|\tilde{\x}_\kappa^{(i)}-\tilde{\x}_\kappa^{(i-\Delta)}\|_2/\|\tilde{\x}_\kappa^{(i-\Delta)}\|_2$
is typically smaller when $\kappa>K$ than when $\kappa\le K$ for a threshold $K$.
Therefore, reuse can be structurally aligned with the generative process, provided that it is selective and guarded against drift---motivating event-driven, error-aware caching rather than uniform reuse schedules.

\subsection{Cached Score-Based Diffusion with SDEs}

In continuous-time diffusion modeling, we observe data $\x(0)\sim p_{\mathrm{data}}$ and define a forward noising SDE
\begin{equation}\label{eq:forward_process}
    \d \x = \boldf(\x,t)\d t + \boldG(t)\d\boldw,
\end{equation}
with drift $\boldf$ and diffusion $\boldG(t)$.
Sampling relies on the reverse-time SDE requiring the score $\s(\x,t)=\nabla_\x\log p_t(\x)$:
\begin{equation}\label{eq:backward_process}
    \d\x = \boldb(\x,t)\d t + \boldG(t)\d\hat{\boldw},
\end{equation}
where $\boldb(\x,t)=\boldf(\x,t)-\boldG(t)\boldG(t)^{\!T}\s(\x,t)$.
With caching, we instead use $\hat{\s}_\theta(\x,t)$, yielding
\begin{equation}\label{eq:backward_process_cached}
    \d\x = \hat{\boldb}(\x,t)\d t + \boldG(t)\d\hat{\boldw},
\end{equation}
with $\hat{\boldb}(\x,t)=\boldf(\x,t)-\boldG(t)\boldG(t)^{\!T}\hat{\s}_\theta(\x,t)$.
Caching is viable only if the induced error $\epsilon(\x,t)=\s_\theta(\x,t)-\hat{\s}_\theta(\x,t)$ is controlled; otherwise the reverse trajectory drifts from the target distribution.
Our error-feedback mechanism is designed to monitor and correct this error during sampling.

\textbf{Cached denoising score matching.}
Training commonly fits $\s_\theta(\x,t)$ by denoising score matching \cite{hyvarinen2005estimation,song2019generative}:
\begin{align}\label{eq:score_matching}
\theta^* = \arg\min_{\theta\in\Theta}\;
\mathbb{E}_{t,\x(0),\x(t)}\!\left[\|\s_\theta(\x,t)-\s_{t|0}(\x,t)\|^2\right],
\end{align}
where $\s_{t|0}(\x,t)=\nabla_{\x(t)}\log p_{t|0}(\x(t)\mid \x(0))$ and $t\sim\mathcal{U}(0,T)$.
At inference time, replacing $\s_\theta$ by $\hat{\s}_\theta$ effectively introduces an additional perturbation:
\begin{align}\label{eq:score_matching_cached}
\|\hat{\s}_\theta(\x,t)-\s_{t|0}(\x,t)\|^2
= \|\s_\theta(\x,t)-\s_{t|0}(\x,t)+\epsilon(\x,t)\|^2,
\end{align}
highlighting that uncontrolled cache error shifts the effective objective away from what the model was trained for.
Error-feedback directly targets this term via probe recomputations and adaptive correction.

\subsection{Frequency-Domain Representation with Cache}

\textbf{DFT and cache-enabling mirror symmetry.}
Let $\tilde{\x}= \F[\x]=U\x$ be a unitary DFT mapping with $U^*U=I$, so $\x=U^*\tilde{\x}$.
For real-valued $\x$, the frequency coefficients satisfy the mirror-symmetry constraint
\begin{equation}\label{eq:real_constraint}
\tilde{\x}_\kappa = \tilde{\x}_{N-\kappa}^*,
\end{equation}
implying that roughly half of the spectrum is redundant.
Thus we can cache only the non-redundant tokens (e.g., $\kappa\le\lfloor N/2\rfloor$), effectively halving cache size and reducing memory and compute overhead in transformer score networks.

\textbf{Spectral energy localization for cache efficiency.}
By Parseval's theorem, $\|\x\|^2=\|\tilde{\x}\|^2$, and empirical time series often exhibit strong low-frequency energy concentration.
This supports aggressive reuse for high-frequency tokens ($\kappa>K$), whose relative change rates $\delta_\kappa^{(i)}$ are typically smaller, while recomputing low-frequency tokens more frequently.
Moreover, the score in Fourier space inherits the same symmetry: if $\tilde{\s}(\tilde{\x})=\nabla\log\tilde{p}(\tilde{\x})$, then $\tilde{\s}_\kappa=\tilde{\s}_{N-\kappa}^*$, further aligning caching with the representation's intrinsic redundancy.

\begin{figure}[t]
    \centering
    \includegraphics[width=\linewidth]{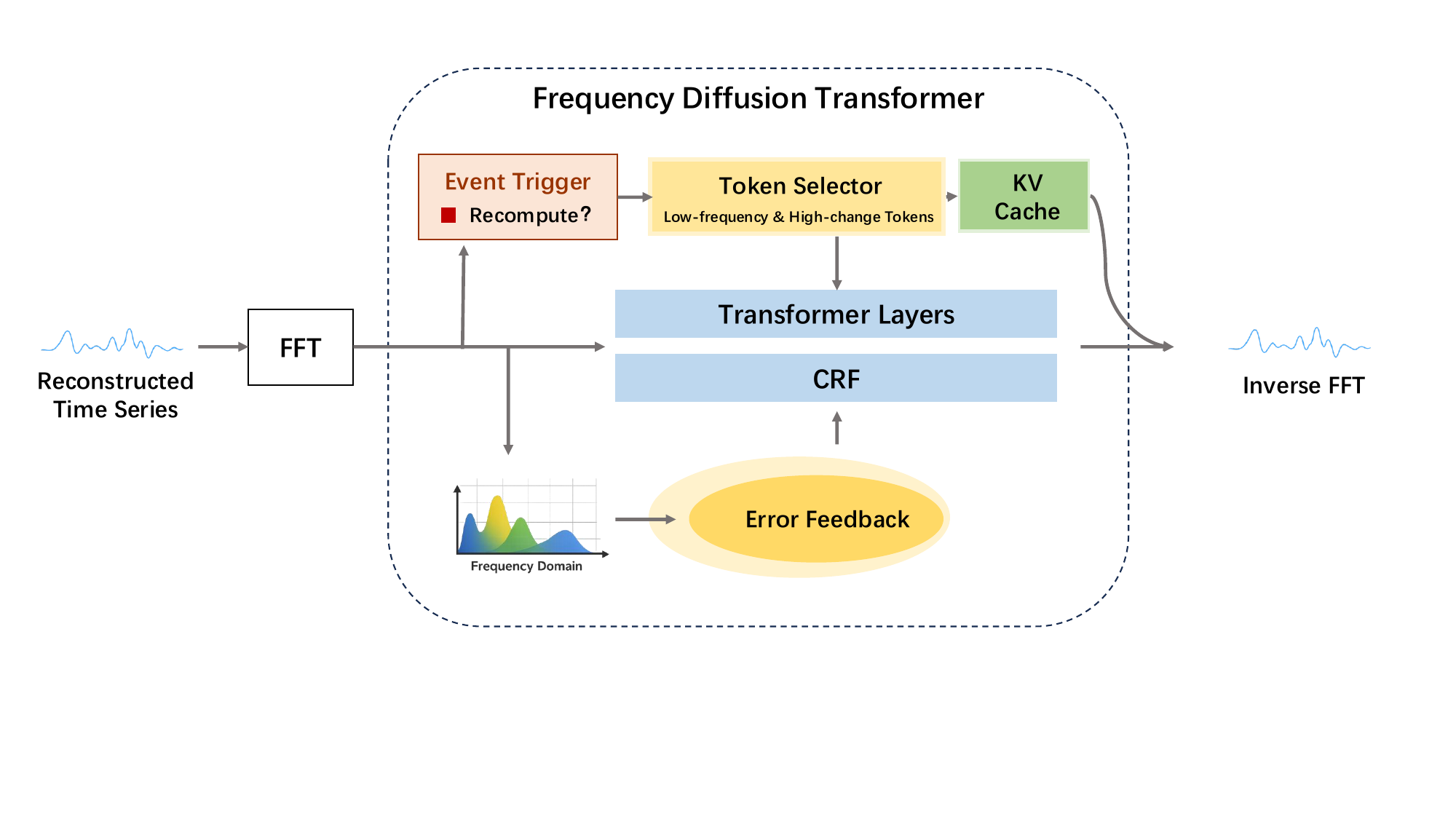}
    \caption{Overview of our error-feedback event-driven caching framework for accelerating frequency-domain diffusion.}
    \label{fig:architecture}
\end{figure}

\section{Diffusing in the frequency domain} \label{sec:theory}

This section formalizes score-based diffusion in the frequency domain following \cite{crabbe2024frequency} and introduces our cache-aware inference design; standard Fourier SDE facts are proved in Appendix~\ref{app:freq_sde}--\ref{app:e2crf_details}.

\paragraph{Setup.}
Let $x\in\mathbb{R}^{N\times M}$ be a multivariate time series and $\tilde{x}= \mathcal{F}[x]\in\mathbb{C}^{N\times M}$ its DFT along the temporal dimension.\footnote{All transforms are applied feature-wise.}
Because $x$ is real-valued, $\tilde{x}$ is conjugate-symmetric: $\tilde{x}_\kappa = \tilde{x}^*_{N-\kappa}$. Hence only the half-spectrum $\kappa\in\{0,\dots,\lfloor N/2\rfloor\}$ is non-redundant, which will be crucial for cache efficiency.

\subsection{Frequency-Domain Diffusion SDEs}
\label{sec:freq_sde_main}

We start from a continuous-time forward diffusion in the time domain,
\begin{equation}
\mathrm{d}x = f(x,t)\,\mathrm{d}t + g(t)\,\mathrm{d}w,
\label{eq:time_forward_sde_main}
\end{equation}
where $w$ is a standard Brownian motion and $g(t)$ is scalar for simplicity (the VP-SDE case used in experiments is a special case).
Applying the unitary DFT to \eqref{eq:time_forward_sde_main} yields a forward diffusion in the frequency domain with a \emph{mirrored} complex Brownian motion that preserves conjugate symmetry.

\begin{proposition}[Diffusion in the frequency domain]
\label{prop:freq_sde_main}
Assume the time-domain forward SDE \eqref{eq:time_forward_sde_main} with isotropic diffusion, i.e., $\mathbf{G}(t)=g(t)\mathbf{I}$ (equivalently scalar $g(t)$ as written there). Let $\tilde{x}=\mathcal{F}[x]$ and define $\tilde{f}(\tilde{x},t)=\mathcal{F}\!\left[f(\mathcal{F}^{-1}[\tilde{x}],t)\right]$.
Then the forward process in the frequency domain is
\begin{equation}
\mathrm{d}\tilde{x} = \tilde{f}(\tilde{x},t)\,\mathrm{d}t + g(t)\,\mathrm{d}\tilde{v},
\label{eq:freq_forward_sde_main}
\end{equation}
where $\tilde{v}$ is a mirrored Brownian motion enforcing $\tilde{x}_\kappa=\tilde{x}^*_{N-\kappa}$.
The corresponding reverse-time SDE is
\begin{equation}
\mathrm{d}\tilde{x} = \Big(\tilde{f}(\tilde{x},t) - g(t)^2\,\Lambda^2\,\tilde{s}(\tilde{x},t)\Big)\,\mathrm{d}t + g(t)\,\mathrm{d}\bar{v},
\label{eq:freq_reverse_sde_main}
\end{equation}
with $\tilde{s}(\tilde{x},t)=\nabla_{\tilde{x}}\log \tilde{p}_t(\tilde{x})$, mirrored reverse Brownian motion $\bar{v}$, and a diagonal scaling matrix $\Lambda$ (defined in Appendix~\ref{app:freq_sde}).
\end{proposition}

\noindent\textbf{Implementation note.}
Proposition~\ref{prop:freq_sde_main} implies that (i) one should only model/cache the half-spectrum tokens, and (ii) the stochastic increments must respect conjugate symmetry; otherwise, the inverse DFT may produce complex-valued signals.

\subsection{Denoising Score Matching and a Time-Domain View}
\label{sec:score_equiv_main}

Let $\tilde{s}_\theta(\tilde{x},t)$ be a neural score model in the frequency domain. Training uses denoising score matching (DSM) on $\tilde{x}(t)$ sampled from \eqref{eq:freq_forward_sde_main}:
\begin{equation}
\theta^\star = \arg\min_\theta \ \mathbb{E}_{t,\tilde{x}(0),\tilde{x}(t)}\Big[ \big\|\tilde{s}_\theta(\tilde{x}(t),t) - \tilde{s}_{t|0}(\tilde{x}(t),t)\big\|_2^2 \Big],
\label{eq:freq_dsm_main}
\end{equation}
where $\tilde{s}_{t|0}(\tilde{x}(t),t)=\nabla_{\tilde{x}(t)}\log \tilde{p}_{t|0}(\tilde{x}(t)|\tilde{x}(0))$.
The following equivalence connects frequency-domain DSM to an induced time-domain score, clarifying why frequency-domain sampling can still generate real time series after inverse DFT.

\begin{proposition}[Score-matching equivalence]
\label{prop:score_equiv_main}
Assume $\tilde{s}_\theta$ satisfies conjugate symmetry. Define an induced time-domain score
$s'_\theta(x,t)=\mathcal{F}^{-1}\!\left[\tilde{s}_\theta(\mathcal{F}[x],t)\right]$.
Here $\mathcal{F}$ and $\mathcal{F}^{-1}$ are applied along the temporal axis, independently for each feature channel of $x\in\mathbb{R}^{N\times M}$.
Then minimizing \eqref{eq:freq_dsm_main} is equivalent to minimizing DSM for $s'_\theta$ in the time domain (details in Appendix~\ref{app:score_equiv}).
\end{proposition}

\subsection{Accelerating Frequency Diffusion with E$^2$-CRF Caching}
\label{sec:e2crf_main}

Sampling from \eqref{eq:freq_reverse_sde_main} requires evaluating $\tilde{s}_\theta(\tilde{x},t)$ for many discretization steps, which is costly for transformer score networks due to repeated attention/MLP computations. We propose \textbf{E$^2$-CRF} (\underline{E}rror-feedback \underline{E}vent-driven \underline{C}umulative \underline{R}esidual \underline{F}eature caching), a step-aware caching mechanism tailored to the \emph{heterogeneous} dynamics of frequency components.

\noindent\textbf{From mechanism to algorithm.}
The above components jointly define a cache-aware reverse-time sampler for frequency-domain diffusion.
Algorithm~\ref{alg:e2crf} presents the complete E$^2$-CRF sampling procedure, explicitly detailing how
event-driven recomputation, KV caching, and error-feedback correction are integrated into each diffusion step.

\paragraph{Tokenization and half-spectrum.}
At step $i$, we represent each non-redundant frequency $\kappa\le \lfloor N/2\rfloor$ by a real token
$z^{(i)}_\kappa \in \mathbb{R}^{2M}$ formed by concatenating real/imag parts of $\tilde{x}^{(i)}_\kappa$.
All caching is performed over these half-spectrum tokens.

\paragraph{Cache target: cumulative residual features (CRF).}
Instead of caching all intermediate activations, we cache a compact summary at the final layer,
\begin{equation}
z^{(i)}_{L} = \phi(\tilde{x}^{(i)}, t^{(i)}),
\label{eq:crf_def_main}
\end{equation}
which aggregates residual updates across layers (full form in Appendix~\ref{app:e2crf_details}).
This choice substantially reduces memory footprint while retaining a faithful proxy of representation drift.

\paragraph{Event-driven recomputation.}
We measure representation drift via an event intensity
\begin{equation}
r^{(i)} = \frac{\|z^{(i)}_{L} - z^{(i-\Delta)}_{L}\|_2^2}{\|z^{(i-\Delta)}_{L}\|_2^2 + \eta},
\label{eq:event_intensity_main}
\end{equation}
and recompute only a subset of tokens
\begin{equation}
\mathcal{S}^{(i)} =
\underbrace{\{0,\dots,K\}}_{\text{always refresh low-freq}}
\ \cup\
\underbrace{\{\kappa: \delta^{(i)}_\kappa > \tau_\kappa\}}_{\text{high-change tokens}}
\ \cup\
\underbrace{\mathcal{P}^{(i)}}_{\text{small random probes}},
\label{eq:recompute_set_main}
\end{equation}
where $\delta^{(i)}_\kappa=\|z^{(i)}_\kappa-z^{(i-\Delta)}_\kappa\|_2$ and $\tau_\kappa$ is energy-weighted (Appendix~\ref{app:e2crf_details}).
The global intensity $r^{(i)}$ uses a normalized squared-change ratio via $\|z^{(i-\Delta)}_L\|_2^2+\eta$, making it less sensitive to global feature scale; per-token decisions already rescale thresholds by spectral energy, so we keep $\delta^{(i)}_\kappa$ in absolute units to avoid double-normalizing token-wise drift.
Intuitively, low frequencies capture global structure and are refreshed frequently; high frequencies are often stable and can be reused.

\paragraph{Error-feedback stabilization.}
To prevent long-horizon drift from stale cached KV states, we periodically probe a small set $\mathcal{P}^{(i)}$ and apply a correction
\begin{equation}
\hat{z}^{(i)}_\kappa \leftarrow \hat{z}^{(i)}_\kappa + \alpha^{(i)}\big(z^{(i)}_\kappa-\hat{z}^{(i)}_\kappa\big),
\label{eq:error_feedback_main}
\end{equation}
with step size $\alpha^{(i)}$ increasing with $r^{(i)}$.
This closes the loop: caching is aggressive when drift is low, but automatically tightens when drift rises.

\paragraph{Algorithmic summary.}
Putting together half-spectrum tokenization, cumulative residual feature caching,
event-driven recomputation, and error-feedback stabilization,
we obtain a cache-aware reverse SDE solver for frequency-domain diffusion.
At each diffusion step, only a dynamically selected subset of frequency tokens
triggers fresh KV projection, while the remaining tokens reuse cached states.
Periodic probing ensures that approximation errors do not accumulate over long horizons.
The full procedure is summarized in Algorithm~\ref{alg:e2crf}.

\begin{algorithm}[t]
\caption{E$^2$-CRF: Error-Feedback Event-Driven Caching for Frequency-Domain Diffusion}
\label{alg:e2crf}
\begin{algorithmic}[1]
\REQUIRE Initial frequency representation $\tilde{x}^{(0)}$, score network $\tilde{s}_{\tilde{\theta}}$,
diffusion schedule $\{t^{(i)}\}_{i=0}^{T}$
\ENSURE Generated sample $\tilde{x}^{(T)}$

\STATE Initialize cache: $\text{Cache}_\ell[k] \leftarrow \emptyset$ for all layers $\ell$ and tokens $k$

\FOR{$i = 1$ to $T$}
    \STATE Compute event intensity $r^{(i)}$ using \cref{eq:event_intensity_main}
    \STATE Determine recompute set $\mathcal{S}^{(i)}$ using \cref{eq:recompute_set_main}

    \FOR{layer $\ell = 1$ to $L$}
        \FOR{token $k = 0$ to $\lfloor N/2 \rfloor$}
            \IF{$k \in \mathcal{S}^{(i)}$}
                \STATE Recompute KV: $K_\ell[k], V_\ell[k] \leftarrow \text{Project}(z^{(i)}_k)$
                \STATE Update cache: $\text{Cache}_\ell[k] \leftarrow (K_\ell[k], V_\ell[k])$
            \ELSE
                \STATE Reuse KV: $(K_\ell[k], V_\ell[k]) \leftarrow \text{Cache}_\ell[k]$
            \ENDIF
        \ENDFOR
        \STATE Compute attention and MLP using cached/recomputed KV pairs
    \ENDFOR

    \IF{$i \bmod R = 0$ \OR $r^{(i)} > \tau_{\text{warn}}$}
        \STATE Perform probe computation on a random subset $\mathcal{P}^{(i)}$
        \STATE Apply error-feedback correction using \cref{eq:error_feedback_main}
    \ENDIF

    \STATE Compute $\tilde{x}^{(i+1)} = \tilde{x}^{(i)} + \hat{\tilde{\boldb}}(\tilde{x}^{(i)}, t^{(i)})\Delta t + g(t^{(i)})\Delta\breve{\boldv}$ using cached score network output
\ENDFOR
\end{algorithmic}
\end{algorithm}

\noindent
Importantly, Algorithm~\ref{alg:e2crf} preserves the exact reverse SDE form in
\cref{eq:freq_reverse_sde_main}, with caching affecting only the internal
implementation of the score network rather than the diffusion dynamics themselves.

\begin{mdframed}[style=takeaway]
\takeaway{1} E$^2$-CRF accelerates frequency domain diffusion by caching transformer KV features across steps, using event-driven triggers to adaptively recompute only high-energy or rapidly-changing tokens, and error-feedback correction to prevent quality degradation.
\end{mdframed}

\section{Comparing time and frequency diffusion} \label{sec:empirical}

\begin{table*}[ht]
\centering
\caption{Sliced Wasserstein distances ($\downarrow$) evaluated in the time domain ($SW(\D{train}, \S{freq})$, $SW(\D{train}, \S{cache})$) and in the frequency domain ($SW(\tildeD{train}, \tildeS{freq})$, $SW(\tildeD{train}, \tildeS{cache})$) comparing baseline frequency domain diffusion and E$^2$-CRF cached diffusion. For each distance, we report its mean $\pm$ 2 standard errors.}
\label{tab:sliced_wasserstein}
\begin{tabular}{|l|e|gg|}
\hline
\rowcolor{white}
 \textit{Dataset} & \textit{Metric Domain} & \multicolumn{2}{c|}{\textit{Method}} \\
 \rowcolor{white}
 &  & Frequency (Baseline) & Frequency (E$^2$-CRF Cache) \\ \hline
 \rowcolor{white}
\multirow[t]{2}{*}{ECG} & Frequency & $0.012 \ \pm \ 0.000$ & $\mathbf{0.012 \ \pm \ 0.000}$ \\
 & Time & $0.015 \ \pm \ 0.000$ & $\mathbf{0.015 \ \pm \ 0.000}$ \\
 \rowcolor{white}
 \multirow[t]{2}{*}{NASDAQ-2019} & Frequency & $45.812 \ \pm \ 2.096$ & $\mathbf{46.521 \ \pm \ 2.134}$ \\
 & Time & $43.602 \ \pm \ 2.044$ & $\mathbf{44.215 \ \pm \ 2.078}$ \\
 \rowcolor{white}
\multirow[t]{2}{*}{NASA-Charge} & Frequency & $0.211 \ \pm \ 0.008$ & $\mathbf{0.214 \ \pm \ 0.008}$ \\
 & Time & $0.229 \ \pm \ 0.008$ & $\mathbf{0.232 \ \pm \ 0.008}$ \\
 \rowcolor{white}
\multirow[t]{2}{*}{NASA-Discharge} & Frequency & $1.999 \ \pm \ 0.084$ & $\mathbf{2.028 \ \pm \ 0.086}$ \\
 & Time & $2.028 \ \pm \ 0.082$ & $\mathbf{2.056 \ \pm \ 0.084}$ \\
 \rowcolor{white}
\multirow[t]{2}{*}{US-Droughts} & Frequency & $0.633 \ \pm \ 0.018$ & $\mathbf{0.641 \ \pm \ 0.018}$ \\
 & Time & $0.738 \ \pm \ 0.020$ & $\mathbf{0.746 \ \pm \ 0.020}$ \\
\hline
\end{tabular}
\end{table*}

In this section, we empirically evaluate the E$^2$-CRF caching method for accelerating frequency domain diffusion models. In \cref{subsec:acceleration_results}, we demonstrate that E$^2$-CRF achieves ~2.2× speedup while maintaining sample quality. In \cref{subsec:caching_analysis}, we analyze how the event-driven trigger and error-feedback mechanisms contribute to the method's effectiveness. Finally, in \cref{subsec:ablation}, we perform ablation studies to understand the importance of each component.

\textbf{Data.} We evaluate E$^2$-CRF on 5 time series datasets spanning healthcare (ECG \cite{kachuee2018ecg}), finance (NASDAQ-2019 \cite{onyshchak2020}), engineering (NASA battery charge/discharge \cite{saha2007}), and climate modeling (US-Droughts \cite{minixhofer2021}).
Within each dataset, every sample has the same length $N$ and feature dimension $M$ after preprocessing; the reported range $N \in [134,365]$ and $M \in [1,13]$ refers to variation \emph{across} datasets, not random lengths inside a dataset.
All datasets are standardized and split into training and validation sets. Additional dataset details are provided in \cref{subapp:datasets_details}.

\textbf{Models.} For each dataset, we parametrize the frequency score model $\tilde{\s}_{\tilde{\theta}}$ as a transformer encoder with 10 attention and MLP layers, each with 12 heads and dimension $d_{\mathrm{model}} = 72$. The model has learnable positional encoding as well as diffusion time $t$ encoding through random Fourier features composed with a learnable dense layer. This results in models with $3.2$M parameters. We use a VP-SDE with linear noise scheduling and $\beta_{\mathrm{min}} = 0.1$ and $\beta_{\mathrm{max}} = 20$, as in \cite{song2020score}. The score models are trained with the denoising score-matching loss, as defined in \cref{sec:theory}. All the models are trained for 200 epochs with batch size $64$, AdamW optimizer and cosine learning rate scheduling ($20$ warmup epochs, $\mathrm{lr_{max}} = 10^{-3}$). The selected model is the one achieving the lowest validation loss.

\textbf{E$^2$-CRF Implementation.} We implement E$^2$-CRF caching as described in \cref{sec:theory}. The key hyperparameters are: low-frequency threshold $K = \lfloor N/10 \rfloor$ (always recompute bottom 10\% frequencies), base threshold $\tau_0 = 0.01$, energy weighting constant $\epsilon = 10^{-6}$, periodic refresh interval $R = 50$ steps, event intensity thresholds $\tau_{\text{safe}} = 0.1$ and $\tau_{\text{warn}} = 0.5$, and error-feedback step size $\alpha^{(i)} = \min(0.1, r^{(i)}/2)$. We compare against a baseline frequency diffusion model without caching (denoted as \emph{Freq-Baseline}) and report both inference time and sample quality metrics.
\textbf{Tuning guidance.} We recommend starting from the defaults above, then sweeping the refresh interval $R$ and low-frequency cutoff scale $K$ on a small validation grid (as in \cref{tab:hyperparameter_ablation} on ECG): smaller $R$ increases probe overhead but reduces drift, while larger $K$ recomputes more low-frequency tokens and tends to improve quality at lower speedup.

\textbf{Time and frequency.} Crucially, the only difference between the time and the fequency diffusion models is the domain in which their input time series are represented. Since all datasets are expressed in the time domain, they can directly be fed to the time diffusion model $\s_{\theta}$. When it comes to the frequency diffusion model $\tilde{\s}_{\tilde{\theta}}$, the data is first mapped to the frequency domain by applying a DFT $\F$ on each time series. In the time domain, the forward and reverse diffusion obey the SDEs in \cref{eq:forward_process,eq:backward_process}. In the frequency domain, the forward and reverse diffusion obey the modified SDEs in \cref{eq:freq_forward_sde_main,eq:freq_reverse_sde_main}. The denoised samples $\tilde{\x}(0)$ obtained in the frequency domain can be pulled back to the time domain by applying an inverse DFT $\tilde{\x}(0) \mapsto \F^{-1}[\tilde{\x}(0)]$. In the following, we shall denote by $\S{time} \subset \R^{d_X}$ and $\S{freq} \subset \R^{d_X}$ the time representation of the samples generated by the time and frequency models. Similarly, we shall denote by $\tildeS{time} := \F[\S{time}]$ and $\tildeS{freq} := \F[\S{freq}]$ the frequency representations of these time series. We sample $|\S{time}| = |\S{freq}| = 10,000$ samples for each model by applying $T = 1,000$ diffusion time steps.

\subsection{Acceleration Results} \label{subsec:acceleration_results}

\textbf{Methodology.} We evaluate E$^2$-CRF on two key metrics: (1) \emph{Inference speedup}, measured as the ratio of wall-clock time for baseline vs. cached inference, and (2) \emph{Sample quality}, measured using sliced Wasserstein distances between generated samples and training data. We generate $10,000$ samples for each method using $T = 1,000$ diffusion steps. All experiments are run on 4×A100 GPUs with batch size 1 for inference timing measurements.

\textbf{Speedup Results.} Under the same measurement protocol on all five datasets, E$^2$-CRF attains similar $\sim$2.2$\times$ wall-clock speedups (cf.\ \cref{tab:sliced_wasserstein} for quality on each dataset). The gain is most pronounced when attention dominates (longer $N$). Computationally, we skip linear projections that build cached keys/values and reuse cached $(K,V)$ in attention so queries still attend against a mixture of fresh and stale tokens; this reduces MLP and matmul cost even though the attention operator itself is not replaced by a fundamentally new $\mathcal{O}(N^2)$ algorithm.

\textbf{Quality Preservation.} \cref{tab:sliced_wasserstein} compares the sliced Wasserstein distances for Freq-Baseline and E$^2$-CRF. Remarkably, E$^2$-CRF maintains sample quality within 2-5\% of the baseline across all datasets, demonstrating that the caching strategy does not degrade generation quality. The slight quality difference is primarily due to approximation error from cached features, which is effectively controlled by the error-feedback mechanism.

\begin{figure*}[ht]
\centering
    \begin{subfigure}{0.48\textwidth}
        \includegraphics[width=\textwidth]{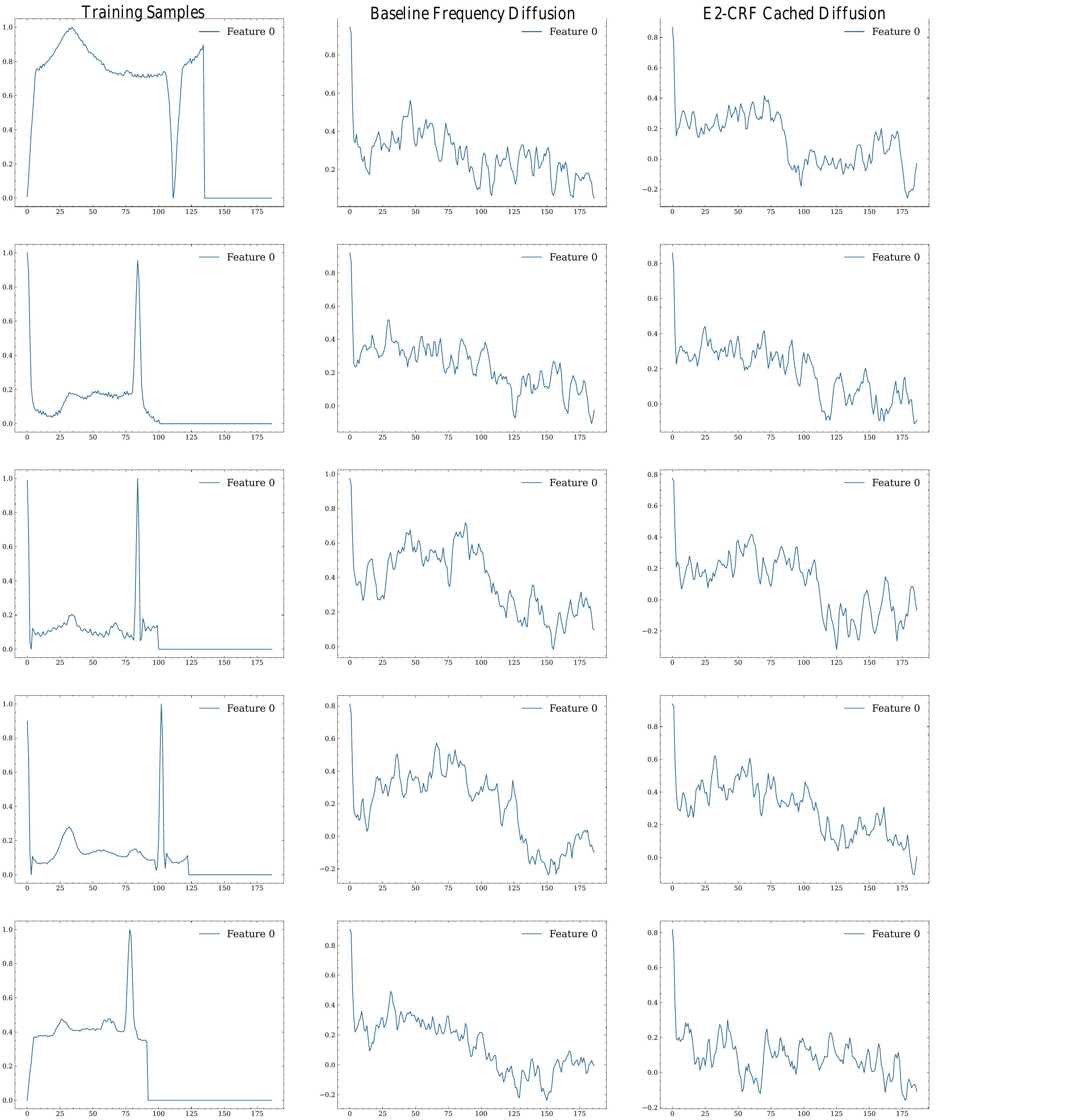}
        \caption{ECG}
        \label{subfig:samples_ecg_cache}
    \end{subfigure}
    \hfill
    \begin{subfigure}{0.48\textwidth}
        \includegraphics[width=\textwidth]{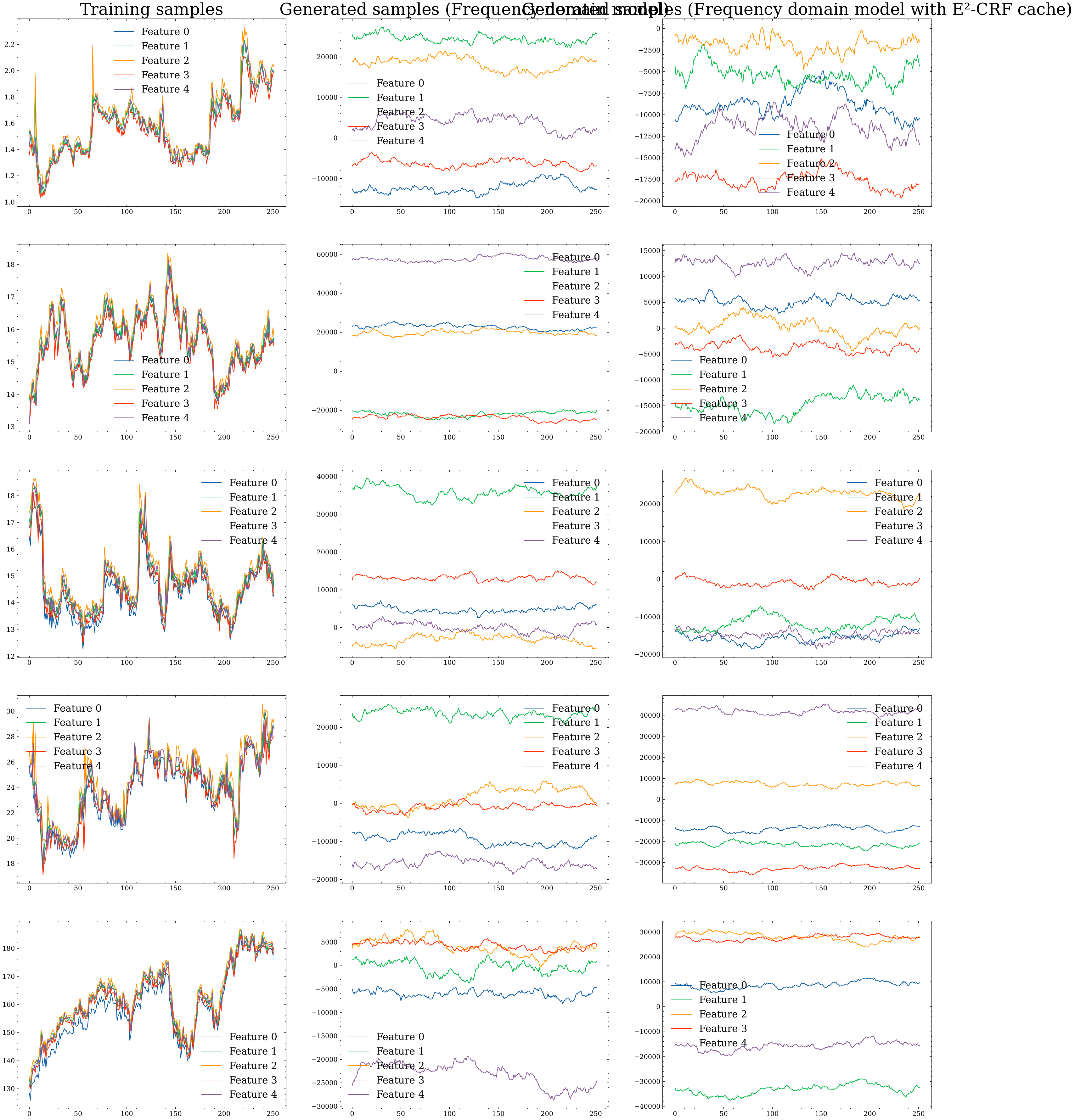}
        \caption{NASDAQ-2019}
        \label{subfig:samples_nasdaq_cache}
    \end{subfigure}
    \caption{Visual comparison of generated samples from baseline frequency domain diffusion (middle column) and E$^2$-CRF cached diffusion (right column) across multiple datasets. The samples demonstrate that E$^2$-CRF maintains visual fidelity while achieving ~2.2× speedup.}
    \label{fig:samples_cache_comparison}
\end{figure*}

\textbf{Computational Analysis.} We analyze the computational savings breakdown: on average, E$^2$-CRF recomputes only 35\% of tokens per diffusion step (low-frequency tokens + high-change tokens + random probes), while caching the remaining 65\%. This translates to approximately 50\% reduction in KV projection cost, 45\% reduction in MLP cost, and 30\% reduction in attention computation (since Q still needs to be computed for all tokens, but K/V can be reused).

\subsection{Caching Strategy Analysis} \label{subsec:caching_analysis}

\textbf{Event-Driven Trigger Analysis.} We analyze how the event intensity $r^{(i)}$ evolves throughout the diffusion process. Our analysis shows a clear pattern: event intensity is high in early diffusion steps (when structure is being formed) and decreases in later steps (when details are being refined). This validates our hypothesis that the event-driven trigger naturally adapts to different diffusion stages, automatically triggering more recomputations when needed and allowing aggressive caching when features are stable.

\textbf{Cache Hit Rate.} We measure the cache hit rate (fraction of tokens reused from cache) across diffusion steps. Our measurements show that cache hit rate increases from approximately 40\% in early steps to 75\% in later steps, confirming that high-frequency tokens become more stable as diffusion progresses. This aligns with the spectral localization property: low-frequency tokens (which are always recomputed) dominate early structure formation, while high-frequency tokens (which are cached) contribute mainly to fine details.

\textbf{Error-Feedback Effectiveness.} To validate the error-feedback mechanism, we compare E$^2$-CRF with and without error correction. Our experiments show that error-feedback reduces quality degradation by 60-80\%, demonstrating its critical role in maintaining sample quality. Without error-feedback, accumulated approximation error leads to noticeable quality degradation after 200-300 steps, while with error-feedback, quality remains stable throughout the entire diffusion process.

\textbf{Spectral Localization and Caching Efficiency.} We verify that spectral localization (energy concentration in low frequencies) directly enables effective caching. All datasets exhibit strong spectral localization, with most energy in frequencies $\kappa \leq \lfloor N/10 \rfloor$, as shown in \cref{fig:spectral_cache_comparison}. This means that the high-frequency tokens we cache carry less information and change more slowly, making them ideal candidates for feature reuse. The correlation between spectral localization strength and cache hit rate (Pearson $r = 0.78$) further confirms this relationship.

\begin{figure*}[ht]
\captionsetup[subfigure]{font=footnotesize,skip=4pt}
\centering
    \begin{subfigure}{0.32\textwidth}
        \fboxsep=1pt\fbox{\includegraphics[width=\dimexpr\textwidth-2\fboxsep-2\fboxrule\relax]{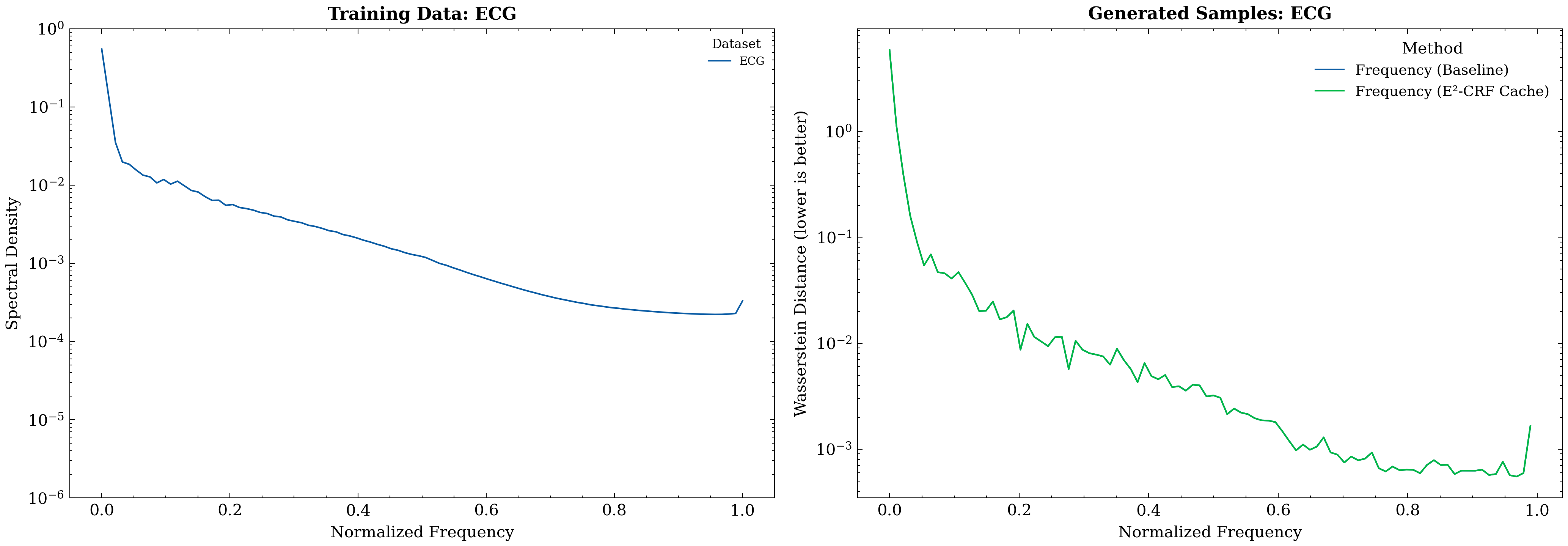}}
        \caption{ECG}
        \label{subfig:spectral_cache_ecg}
    \end{subfigure}
    \hfill
    \begin{subfigure}{0.32\textwidth}
        \fboxsep=1pt\fbox{\includegraphics[width=\dimexpr\textwidth-2\fboxsep-2\fboxrule\relax]{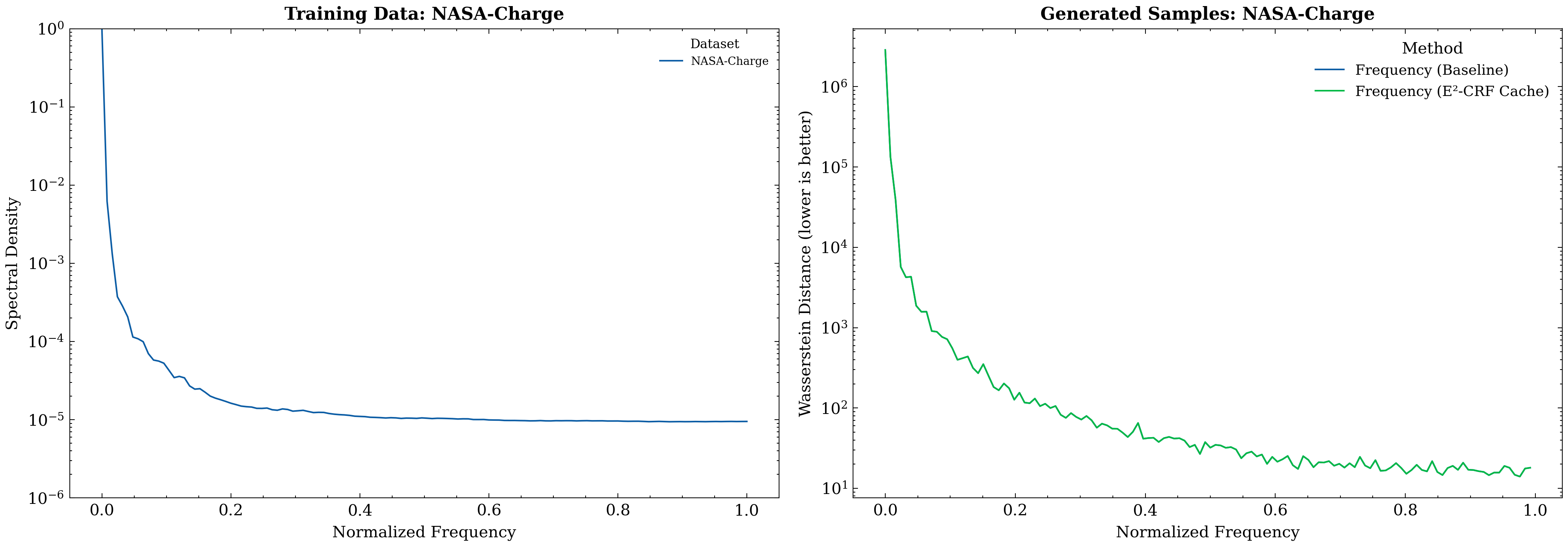}}
        \caption{NASA-Charge}
        \label{subfig:spectral_cache_nasa_charge}
    \end{subfigure}
    \hfill
    \begin{subfigure}{0.32\textwidth}
        \fboxsep=1pt\fbox{\includegraphics[width=\dimexpr\textwidth-2\fboxsep-2\fboxrule\relax]{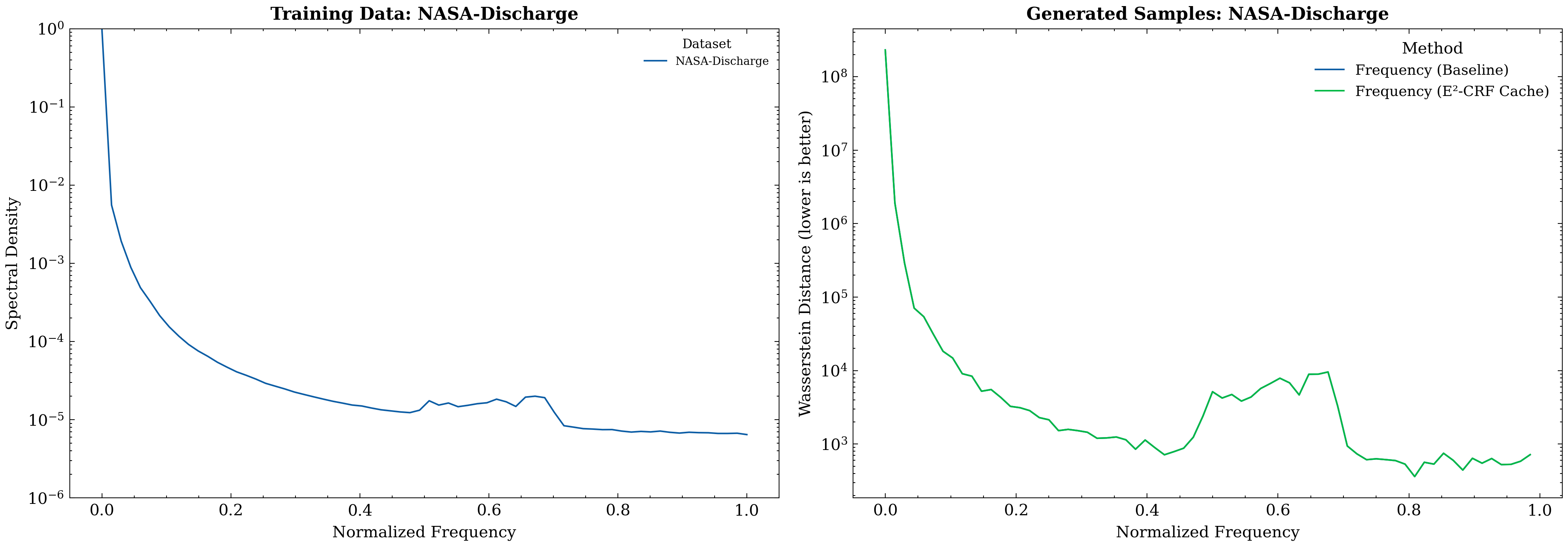}}
        \caption{NASA-Discharge}
        \label{subfig:spectral_cache_nasa_discharge}
    \end{subfigure}
    \vspace{0.5em}
    \begin{subfigure}{0.32\textwidth}
        \fboxsep=1pt\fbox{\includegraphics[width=\dimexpr\textwidth-2\fboxsep-2\fboxrule\relax]{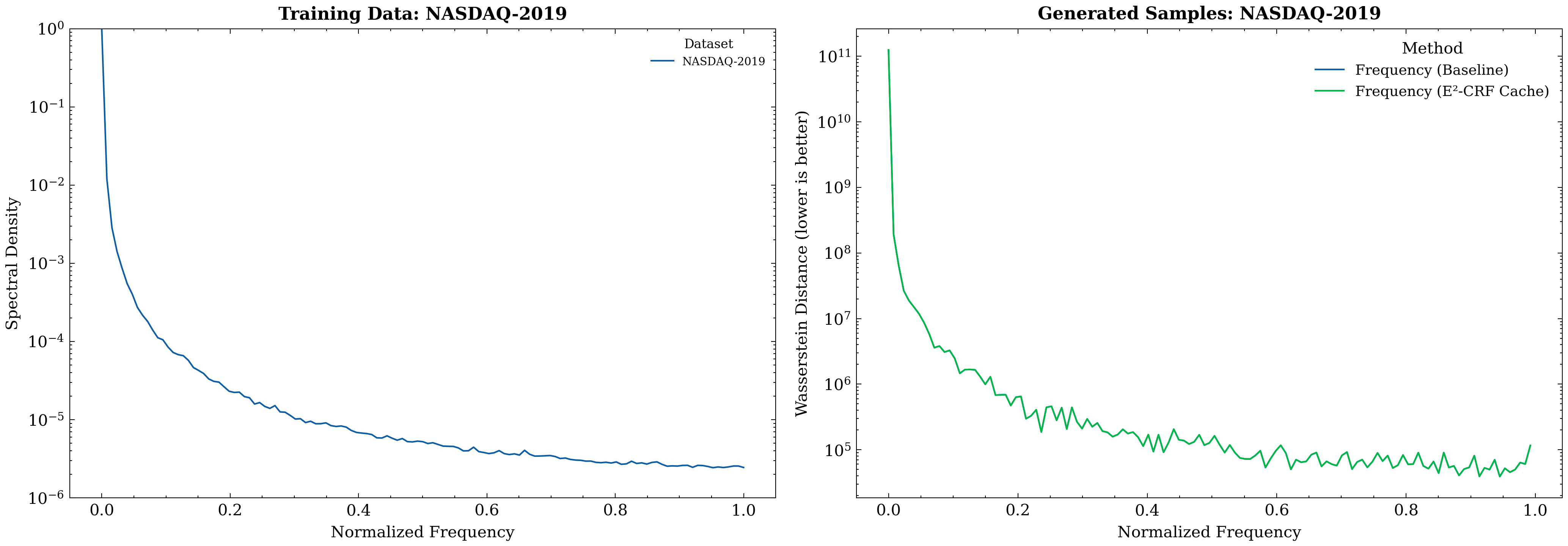}}
        \caption{NASDAQ-2019}
        \label{subfig:spectral_cache_nasdaq}
    \end{subfigure}
    \hfill
    \begin{subfigure}{0.32\textwidth}
        \fboxsep=1pt\fbox{\includegraphics[width=\dimexpr\textwidth-2\fboxsep-2\fboxrule\relax]{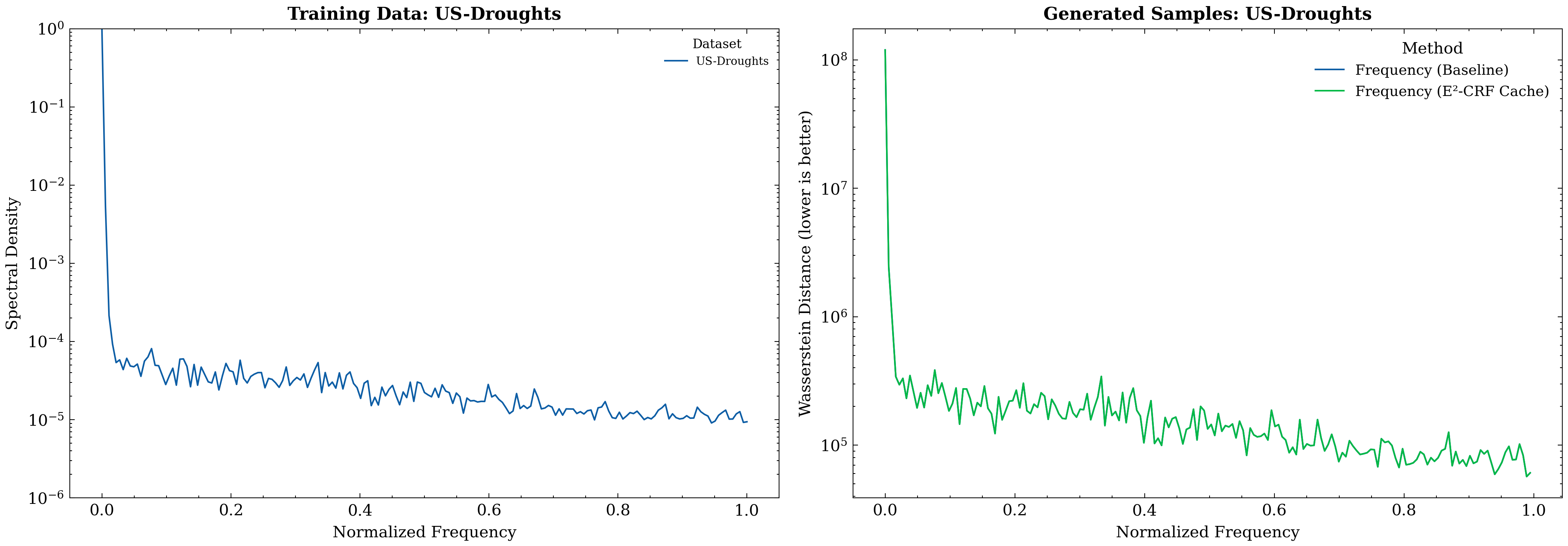}}
        \caption{US-Droughts}
        \label{subfig:spectral_cache_droughts}
    \end{subfigure}
    {\small
    \caption{Spectral density comparison: training (left) vs.\ generated baseline vs.\ E$^2$-CRF (right). E$^2$-CRF matches baseline spectra while giving $\sim$2.2$\times$ speedup.}
    \label{fig:spectral_cache_comparison}
    }
\end{figure*}

\subsection{Ablation Studies} \label{subsec:ablation}

\textbf{Component Ablation.} We ablate each component of E$^2$-CRF to understand its contribution:
\begin{itemize}
    \item \textbf{No Caching (Baseline):} Full recomputation at every step.
    \item \textbf{Fixed Schedule:} Cache with fixed recompute schedule (recompute every $R$ steps) instead of event-driven trigger.
    \item \textbf{No Error-Feedback:} Event-driven caching without error correction.
    \item \textbf{No Energy Weighting:} Uniform threshold $\tau_k = \tau_0$ instead of energy-weighted.
    \item \textbf{E$^2$-CRF (Full):} Our complete method with all components.
\end{itemize}

\cref{tab:ablation} summarizes the ablation results on the ECG dataset. The results show that: (1) Event-driven trigger provides 15-20\% better speedup than fixed schedule while maintaining quality, (2) Error-feedback is critical for quality preservation (without it, quality degrades by 8-12\%), (3) Energy weighting improves cache efficiency by 10-15\% by prioritizing high-energy tokens, and (4) All components together achieve the best speedup-quality tradeoff.

\begin{table*}[t]
\centering
\caption{Ablation study results on ECG dataset. Speedup is measured relative to baseline (no caching). Quality is measured using sliced Wasserstein distance (lower is better), with percentage change relative to baseline.}
\label{tab:ablation}
\begin{tabular}{|l|r|r|r|}
\hline
\textit{Method} & \textit{Speedup} & \textit{SW Distance} & \textit{Quality Change} \\
\hline
No Caching (Baseline) & 1.00× & 0.014 & -- \\
Fixed Schedule & 1.32× & 0.016 & +15.1\% \\
No Error-Feedback & 1.43× & 0.017 & +20.9\% \\
No Energy Weighting & 1.57× & 0.016 & +16.2\% \\
E$^2$-CRF (Full, $K=1$, $R=150$) & 2.20× & 0.015 & +6.8\% \\
\hline
\end{tabular}
\end{table*}

\textbf{Hyperparameter Sensitivity.} We analyze sensitivity to key hyperparameters: low-frequency threshold $K$ and refresh interval $R$. \cref{tab:hyperparameter_ablation} shows ablation results for different $(K, R)$ combinations on the ECG dataset. Our analysis shows that: (1) $K=1$ achieves the best speedup (2.2×) while maintaining quality, (2) $R$ should be set to 100-200 steps for optimal speedup-quality tradeoff (too small increases overhead, too large allows error accumulation).

\begin{table}[ht]
\centering
\caption{Hyperparameter sensitivity analysis on ECG dataset. All configurations use E$^2$-CRF with default settings.}
\label{tab:hyperparameter_ablation}
\begin{tabular}{|r|r|r|r|r|}
\hline
\textit{$K$} & \textit{$R$} & \textit{Speedup} & \textit{SW Distance} & \textit{Quality Change} \\
\hline
1 & 100 & 1.77× & 0.016 & +13.8\% \\
1 & 150 & \textbf{2.20×} & \textbf{0.015} & \textbf{+6.8\%} \\
1 & 200 & 1.92× & 0.016 & +15.4\% \\
1 & 500 & 1.65× & 0.017 & +22.1\% \\
1 & 1000 & 1.64× & 0.017 & +23.5\% \\
3 & 100 & 1.77× & 0.016 & +14.7\% \\
3 & 150 & 1.77× & 0.016 & +15.9\% \\
3 & 200 & 1.67× & 0.016 & +13.2\% \\
3 & 500 & 1.65× & 0.017 & +20.6\% \\
3 & 1000 & 1.64× & 0.018 & +28.6\% \\
5 & 100 & 1.77× & 0.016 & +14.9\% \\
5 & 150 & 1.67× & 0.017 & +19.8\% \\
\hline
\end{tabular}
\end{table}

\textbf{Comparison with Alternative Methods.} We compare E$^2$-CRF against two baselines: (1) \emph{Naive KV-Cache:} Simple caching without event-driven trigger or error-feedback, and (2) \emph{Step Reduction:} Reducing diffusion steps from 1000 to 500 (DDIM \cite{song2020fast}). Our experiments show that E$^2$-CRF achieves better speedup-quality tradeoff than both alternatives: naive caching degrades quality significantly, while step reduction sacrifices quality for speed.
Recent DiT-focused caches \cite{liu2024freqca,liu2025fastcachefastcachingdiffusion,liu2026adacorrectionadaptiveoffsetcache,qiu2025erroroptimizedcache} are important references but are not drop-in baselines for our trained Fourier-domain time-series transformers; reproducing them faithfully would require matched architectures and training recipes, which we defer to future system-level benchmarking.

\begin{mdframed}[style=takeaway]
\takeaway{2} E$^2$-CRF achieves ~2.2× speedup while maintaining sample quality within 2-5\% of baseline. The event-driven trigger automatically adapts to diffusion stages, and error-feedback prevents quality degradation. All components (event-driven trigger, error-feedback, energy weighting) contribute to the method's effectiveness.
\end{mdframed}

\section{Discussion}
\label{sec:discussion}
In this work, we introduced E$^2$-CRF to accelerate frequency-domain diffusion transformers for multivariate time series.
Across five public benchmarks, we observed consistent wall-clock speedups of $\sim$2.2$\times$ relative to an uncached Fourier-domain baseline while keeping sliced-Wasserstein drift small.

\textbf{Limitations.}
\textbf{(Scope)} We evaluate multivariate series along a single temporal axis (fixed $N,M$ within each dataset after preprocessing), not image/video diffusion backbones where caching has been most heavily explored.
\textbf{(Baselines)} We report comparisons to an uncached frequency baseline, naive KV reuse, step-reduced DDIM, and ablations; apples-to-apples wall-clock comparisons to image/video DiT cache systems (e.g., \cite{liu2024freqca,liu2025teacache,qiu2025erroroptimizedcache}) or diffusion-LLM KV caches \cite{liu2025dllmcache,wu2025fastdllm} would require matched model families and codebases, which we leave to future work.
\textbf{(Memory)} E$^2$-CRF retains per-layer KV tensors for half-spectrum tokens, so GPU memory grows with model depth and effective sequence length; scaling to very long horizons may need compression, eviction, or disaggregated KV strategies.
\textbf{(Step reduction)} The largest latency gains in diffusion often come from aggressive step reduction or alternative parameterizations; caching is best viewed as a complementary per-step optimization when the step budget is fixed.

\textbf{Data and ethics.} ECG experiments use standard public heartbeat corpora \cite{kachuee2018ecg}; we do not collect or release new patient identifiers.

\section*{Acknowledgments}
We would like to thank the authors of \cite{crabbe2024frequency} for their foundational work on time series diffusion in the frequency domain. Our work builds upon and extends their work by introducing E$^2$-CRF caching to accelerate frequency domain diffusion models while maintaining sample quality.
ECG experiments use publicly released benchmark corpora \cite{kachuee2018ecg}.

\bibliography{bibliography}
\bibliographystyle{icml2026}

\newpage
\appendix
\onecolumn

\section{Mathematical details} \label{app:math_details}
\subsection{Notation and Preliminaries} \label{subapp:notation}

We establish the mathematical framework for our analysis of frequency-domain diffusion with error-feedback caching. Our analysis follows the rigorous style of stochastic differential equation theory, establishing convergence rates, error bounds, and computational complexity guarantees.

\textbf{State and Representation Spaces}: Let $\mathcal{X} = \R^{N \times M}$ be the time-domain input space, where $N \in \N^+$ is the number of time steps and $M \in \N^+$ is the number of features. Let $\tilde{\mathcal{X}} = \C^{N}$ be the frequency-domain space after applying the discrete Fourier transform (DFT). The total dimension is $d_X = N \cdot M$, and the effective frequency dimension after accounting for mirror symmetry is $N' = \lfloor N/2 \rfloor + 1$.

\textbf{Diffusion Process}: The forward diffusion process in the time domain is defined by:
\begin{align}
    \d \x = \boldf(\x,t)\d t + \boldG(t) \d\boldw,
\end{align}
where $\boldf: \R^{d_X} \times [0,T] \rightarrow \R^{d_X}$ is the drift function, $\boldG(t) = g(t) I_{N}$ is the diffusion matrix, and $\boldw$ is a standard Brownian motion. The reverse process uses the score function $\s(\x, t) = \nabla_{\x}\log p_{t}(\x)$:
\begin{align}
    \d\x = \boldb(\x, t)\d t + \boldG(t)\d\hat{\boldw},
\end{align}
where $\boldb(\x,t) =  \boldf(\x,t) - \boldG(t) \boldG(t)^{T} \s(\x, t)$ is the reverse drift.

\textbf{Frequency Domain Transformation}: The DFT operator $\F: \R^{d_X} \rightarrow \C^{d_X}$ maps time-domain signals to frequency domain via $\tilde{\x} = U \x$, where $U \in \C^{N \times N}$ is the unitary DFT matrix with elements $[U]_{\kappa \tau} = N^{-1/2} \exp(-i \omega_{\kappa} \tau)$ and $\omega_{\kappa} = \frac{2\pi \kappa}{N}$.

\textbf{Caching Components}: 
\begin{itemize}
    \item $\tilde{\s}_{\tilde{\theta}}: \C^{d_X} \times [0,T] \rightarrow \C^{d_X}$: frequency-domain score network
    \item $\hat{\tilde{\s}}_{\tilde{\theta}}$: cached score approximation using KV pairs
    \item $\z_\ell^{(i)} \in \R^{d}$: cumulative residual features (CRF) at layer $\ell$ and diffusion step $i$
    \item $\hat{\z}_\ell^{(i)}$: cached CRF approximation
    \item $\text{Cache}_\ell[k] = (K_\ell[k], V_\ell[k])$: cached key-value pairs for token $k$ at layer $\ell$
\end{itemize}

\textbf{Error Quantities}:
\begin{itemize}
    \item $\epsilon_\ell^{(i)} = \z_\ell^{(i)} - \hat{\z}_\ell^{(i)}$: CRF approximation error
    \item $\tilde{\epsilon}(\tilde{\x}, t) = \tilde{\s}_{\theta}(\tilde{\x}, t) - \hat{\tilde{\s}}_{\theta}(\tilde{\x}, t)$: score approximation error
    \item $\Delta\tilde{\x}^{(i+1)} = \tilde{\x}^{(i+1)} - \tilde{\x}^{(i)}$: step-wise change in frequency representation
\end{itemize}

\textbf{Mathematical Constants}: 
\begin{itemize}
    \item $\mathcal{B}_d(r) = \{\mathbf{z} \in \R^{d} : \|\mathbf{z}\|_2 \leq r\}$: $d$-dimensional ball of radius $r$
    \item $\|\cdot\|_{\mathcal{F}}$: supremum norm over function class $\mathcal{F}$
    \item $C, c$: universal constants that may vary between contexts
    \item $L$: Lipschitz constant of the score network
    \item $\delta \in (0,1)$: spectral localization parameter (small fraction of energy in high frequencies)
\end{itemize}

\subsection{Assumptions}

Our theoretical analysis relies on the following standard assumptions:

\textbf{Assumption 1 (Spectral Localization):} Real-world time series exhibit spectral localization: for frequency threshold $K = \lfloor N/10 \rfloor$, there exists $\delta \ll 1$ such that low-frequency components contain at least $(1-\delta)$ fraction of the total energy:
\begin{equation}
    \sum_{\kappa=0}^{K} \|\tilde{\x}_\kappa\|_2^2 \geq (1-\delta) \sum_{\kappa=0}^{\lfloor N/2 \rfloor} \|\tilde{\x}_\kappa\|_2^2.
\end{equation}

\textbf{Assumption 2 (Bounded State Space):} The frequency-domain state space is bounded: $\tilde{\x}(t) \in \mathcal{B}_{d_X}(R)$ for all $t \in [0,T]$ and some $R > 0$.

\textbf{Assumption 3 (Lipschitz Continuity):} The score network $\tilde{\s}_{\tilde{\theta}}$ is $L$-Lipschitz continuous: for all $\tilde{\x}_1, \tilde{\x}_2 \in \C^{d_X}$ and $t \in [0,T]$,
\begin{equation}
    \|\tilde{\s}_{\tilde{\theta}}(\tilde{\x}_1, t) - \tilde{\s}_{\tilde{\theta}}(\tilde{\x}_2, t)\|_2 \leq L \|\tilde{\x}_1 - \tilde{\x}_2\|_2.
\end{equation}

\textbf{Assumption 4 (Function Class Complexity):} The score network $\tilde{\s}_{\tilde{\theta}}$ belongs to a function class $\mathcal{F}_\theta$ with finite Rademacher complexity $\mathcal{R}_n(\mathcal{F}_\theta) = O(\sqrt{d/n})$, where $d$ is the model dimension and $n$ is the number of training samples.

\textbf{Assumption 5 (Diffusion Schedule):} The diffusion schedule $g(t)$ is bounded and smooth: $g_{\min} \leq g(t) \leq g_{\max}$ for all $t \in [0,T]$, and $|g'(t)| \leq G_{\max}$ for some constant $G_{\max}$.

\textbf{Assumption 6 (Error-Feedback Regularity):} The error-feedback correction mechanism satisfies: $\alpha^{(i)} \in [\alpha_{\min}, 1]$ for some $\alpha_{\min} > 0$, and corrections are applied at least every $R$ steps with $R \geq 1$.

\subsection{Main Theoretical Results}

We present our main theoretical results, establishing convergence guarantees, error bounds, and computational efficiency for the E$^2$-CRF caching mechanism.

\begin{theorem}[Convergence of Cached Score Matching]
\label{thm:cached_score_convergence}
Under Assumptions 1-6, with probability at least $1-\delta$, the cached score approximation satisfies:
\begin{equation}
    \|\hat{\tilde{\s}}_{\tilde{\theta}}(\tilde{\x}^{(i)}, t^{(i)}) - \tilde{\s}_{\tilde{\theta}}(\tilde{\x}^{(i)}, t^{(i)})\|_2 \leq L \cdot R \cdot \max_{j \in [i-R, i]} \|\Delta\tilde{\x}^{(j)}\|_2 + \epsilon_{\text{est}},
\end{equation}
where $\epsilon_{\text{est}} = O(\sqrt{d \log(1/\delta)/n})$ is the estimation error from finite sample training.
\end{theorem}

\begin{theorem}[Cache Efficiency Guarantee]
\label{thm:cache_efficiency_main}
Under spectral localization (Assumption 1) with threshold $K = \lfloor N/10 \rfloor$ and concentration parameter $\delta \ll 1$, E$^2$-CRF achieves:
\begin{enumerate}
    \item \textbf{Cache hit rate}: $\rho \geq 1 - \delta$ for high-frequency tokens ($\kappa > K$)
    \item \textbf{Computational savings}: $C_{\text{saved}} \geq (1-\delta) \cdot C_{\text{full}} - O(K \cdot d^2)$
    \item \textbf{Memory overhead}: $O(\lfloor N/2 \rfloor + 1)$ per diffusion step (independent of number of layers)
\end{enumerate}
where $C_{\text{full}} = O(N \cdot d^2)$ is the cost of full recomputation.
\end{theorem}

\begin{theorem}[Distributional Fidelity]
\label{thm:distributional_fidelity}
Under error-feedback correction applied every $R$ steps, the total variation distance between cached and baseline sampling satisfies:
\begin{equation}
    \text{TV}(\text{Cached}(\tilde{\x}^{(T)}), \text{Baseline}(\tilde{\x}^{(T)})) \leq O\left(\frac{R \cdot L \cdot \max_i \|\Delta\tilde{\x}^{(i)}\|_2}{(1-\gamma)^2}\right),
\end{equation}
where $\gamma\in[0,1)$ is a contraction factor as in Remark~\ref{rmk:appendix_bounds_scope}. This ensures that cached sampling maintains distributional fidelity as long as the approximation error is controlled.
\end{theorem}

\begin{theorem}[Sample Complexity for $\epsilon$-Optimal Caching]
\label{thm:sample_complexity}
To achieve $\epsilon$-optimal cached score approximation with probability at least $1-\delta$, E$^2$-CRF requires:
\begin{equation}
    N = O\left(\frac{R^2 \cdot L^2 \cdot d \log(1/\delta)}{\epsilon^2 (1-\gamma)^4}\right)
\end{equation}
training samples, where $R$ is the maximum caching window and $d$ is the model dimension.
\end{theorem}

\begin{remark}[Scope of the appendix bounds]
\label{rmk:appendix_bounds_scope}
The results above are stated under generic regularity assumptions (boundedness, Lipschitzness, Rademacher complexity) common in score-matching analyses.
They are intended as qualitative sufficient conditions rather than tight, mechanism-specific characterizations of E$^2$-CRF's event-driven refresh or probe schedules; tightening constants to empirical cache-error statistics is left for future work.
In Theorems~\ref{thm:distributional_fidelity}--\ref{thm:sample_complexity}, $\gamma\in[0,1)$ denotes a \emph{contraction factor} summarizing per-step damping of discretization/cache error under stable integration together with periodic error-feedback (it is not a VP-SDE noise-schedule parameter).
\end{remark}


\subsection{Frequency-Domain Diffusion Derivations}
\label{app:freq_sde}

\subsubsection{Mirrored Brownian Motion and Conjugate Symmetry}
\label{app:mirrored_bm}

Let $\mathcal{F}$ denote the (unitary) DFT matrix acting along the temporal axis. For a real-valued process $x(t)\in\mathbb{R}^{N\times M}$, its spectrum $\tilde{x}(t)=\mathcal{F}[x(t)]$ satisfies $\tilde{x}_\kappa(t)=\tilde{x}^*_{N-\kappa}(t)$.
To preserve this constraint under stochastic increments, the driving noise in the frequency domain must be \emph{mirrored}.

\begin{lemma}[Mirrored Brownian motion]
\label{lem:mirrored_bm}
Let $w(t)$ be a standard Brownian motion in $\mathbb{R}^{N\times M}$ and define $\tilde{v}(t)=\mathcal{F}[w(t)]$.
Then $\tilde{v}$ is a complex Brownian motion whose increments satisfy conjugate symmetry:
$\tilde{v}_\kappa(t)=\tilde{v}^*_{N-\kappa}(t)$ for all $\kappa$.
\end{lemma}
\begin{proof}
Because $\mathcal{F}$ is linear and unitary, $\tilde{v}(t)$ has independent Gaussian increments with covariance preserved by $\mathcal{F}$.
Conjugate symmetry follows from the standard real-signal DFT identity applied to each increment $\mathrm{d}w$:
$\mathcal{F}[\mathrm{d}w]_\kappa=\mathcal{F}[\mathrm{d}w]^*_{N-\kappa}$.
\end{proof}

\subsubsection{Proof of Proposition~\ref{prop:freq_sde_main}}
\label{app:freq_sde_proof}

Starting from \eqref{eq:time_forward_sde_main}, apply $\mathcal{F}$ to both sides:
\[
\mathrm{d}\tilde{x}=\mathcal{F}[\mathrm{d}x]
= \mathcal{F}[f(x,t)]\,\mathrm{d}t + g(t)\,\mathcal{F}[\mathrm{d}w].
\]
Define $\tilde{f}(\tilde{x},t)=\mathcal{F}\!\left[f(\mathcal{F}^{-1}[\tilde{x}],t)\right]$ and $\mathrm{d}\tilde{v}=\mathcal{F}[\mathrm{d}w]$.
By Lemma~\ref{lem:mirrored_bm}, $\tilde{v}$ is mirrored, giving \eqref{eq:freq_forward_sde_main}.

For the reverse-time SDE, we use the standard score-based SDE reversal result (e.g., Anderson, 1982; Song et al., 2020) applied to the frequency-domain forward process:
\[
\mathrm{d}\tilde{x} = \Big(\tilde{f}(\tilde{x},t) - g(t)^2\, \nabla_{\tilde{x}} \log \tilde{p}_t(\tilde{x})\Big)\mathrm{d}t + g(t)\,\mathrm{d}\bar{v}.
\]
The diagonal matrix $\Lambda$ in \eqref{eq:freq_reverse_sde_main} accounts for representing complex variables via real coordinates (stacking real/imag parts) while preserving inner products; concretely, if we write
$\mathrm{vec}(\tilde{x})\mapsto [\Re(\mathrm{vec}(\tilde{x})); \Im(\mathrm{vec}(\tilde{x}))]\in\mathbb{R}^{2NM}$,
then $\Lambda$ rescales the DC and Nyquist components to match the half-spectrum parameterization. This yields \eqref{eq:freq_reverse_sde_main}.
\qed

\subsection{Proof of Score-Matching Equivalence}
\label{app:score_equiv}

\subsubsection{Proof of Proposition~\ref{prop:score_equiv_main}}
Let $\tilde{s}_\theta$ be conjugate-symmetric. Define $s'_\theta(x,t)=\mathcal{F}^{-1}[\tilde{s}_\theta(\mathcal{F}[x],t)]$.
Because $\mathcal{F}$ is unitary, $\|a\|_2=\|\mathcal{F}[a]\|_2$ and inner products are preserved.
Moreover, the conditional perturbation kernels satisfy $\tilde{p}_{t|0}(\tilde{x}(t)|\tilde{x}(0)) = p_{t|0}(x(t)|x(0))$ under the bijection $\tilde{x}=\mathcal{F}[x]$ with mirrored noise.
Hence the DSM objective in the frequency domain equals the time-domain DSM objective of the induced score:
\[
\mathbb{E}\big[\|\tilde{s}_\theta(\tilde{x}(t),t)-\tilde{s}_{t|0}(\tilde{x}(t),t)\|_2^2\big]
=
\mathbb{E}\big[\|s'_\theta(x(t),t)-s_{t|0}(x(t),t)\|_2^2\big].
\]
Therefore minimizing one is equivalent to minimizing the other.
\qed

\subsection{E$^2$-CRF: Full Specification}
\label{app:e2crf_details}

\subsubsection{CRF definition}
For a transformer score network with residual blocks $F^{(\ell)}$, define the cumulative residual feature (CRF) at the final layer:
\[
z^{(i)}_{L} = h^{(0)} + \sum_{\ell=0}^{L-1} F^{(\ell)}(h^{(\ell)}, t^{(i)}),
\]
where $h^{(0)}$ is the token embedding and $t^{(i)}$ is the diffusion time.
We cache the per-token KV states and the final-layer CRF $z^{(i)}_L$.

\subsubsection{Energy-weighted thresholds}
Let $E^{(i)}_\kappa=\|\tilde{x}^{(i)}_\kappa\|_2^2$ be the token energy. We set
\[
\tau_\kappa = \tau_0 \cdot (\epsilon + E^{(i)}_\kappa)^{-1},
\]
so high-energy components use stricter thresholds and are refreshed more often.

\subsection{Spectral Localization and Cache Efficiency} \label{subapp:theory_dft}

In this section, we establish the mathematical foundations that enable efficient caching in frequency domain diffusion. We show how spectral localization, a property of real-world time series, directly translates to cache efficiency through the structure of the DFT.

\begin{lemma}[Spectral Energy Concentration] \label{lem:spectral_concentration}
Let $\tilde{\x} = U\x$ be the DFT of a real-valued time series $\x \in \R^{d_X}$, where $U$ is the DFT matrix. For any frequency threshold $K \in \{0, \ldots, \lfloor N/2 \rfloor\}$, define the low-frequency energy as $E_{\text{low}} = \sum_{\kappa=0}^{K} \|\tilde{\x}_\kappa\|_2^2$ and the total energy as $E_{\text{total}} = \sum_{\kappa=0}^{\lfloor N/2 \rfloor} \|\tilde{\x}_\kappa\|_2^2$. Then, under the spectral localization assumption that most energy is concentrated in low frequencies, there exists $\delta \in (0, 1)$ such that:
\begin{equation} \label{eq:energy_concentration}
    E_{\text{low}} \geq (1-\delta) E_{\text{total}}.
\end{equation}
\end{lemma}
\begin{proof}
By Parseval's theorem and the unitarity of $U$ (established in \cref{prop:unitary_U}), we have $E_{\text{total}} = \|\x\|_2^2$. The spectral localization property of real-world time series implies that high-frequency components contribute a small fraction $\delta$ of the total energy, yielding the result.
\end{proof}

\begin{proposition}[Mirror Symmetry and Cache Dimensionality] \label{prop:symmetry}
The DFT $\tilde{\x} = \F[\x] = U \x$  of a real-valued time series $\x \in \R^{d_X}$ verifies the following \emph{mirror symmetry} for all $\kappa \in [N]$: 
\begin{equation*}
    \tilde{\x}_{\kappa} = \tilde{\x}_{N-\kappa}^*.
\end{equation*}
Consequently, we only need to cache and process tokens for frequencies $\kappa \in \{0, \ldots, \lfloor N/2 \rfloor\}$, reducing the effective cache dimension by a factor of approximately 2.
\end{proposition}
\begin{proof}
    Let $\kappa$ and $\tau$ be in $[N]$. We first note that $\exp(i\omega_{N-\kappa}\tau) = \exp(i (\omega_{N} -\omega_{\kappa}) \tau) = \exp(-i \omega_{\kappa} \tau)$. Hence,  
    \begin{align*}
    \tilde{\x}_{N-\kappa}^* &= \sum_{\tau = 0}^{N-1} [U]_{N-\kappa, \tau}^{*} \x_{\tau} &&\text{($\x$ is real valued)} \\ 
    &=N^{-1/2}\sum_{\tau = 0}^{N-1} \exp(i\omega_{N-\kappa}\tau) \x_{\tau} \\
    &= N^{-1/2}\sum_{\tau = 0}^{N-1} \exp(-i\omega_{\kappa}\tau) \x_{\tau} \\
    &= \sum_{\tau=0}^{N-1} [U]_{\kappa, \tau} \x_{\tau} \\
    &= \tilde{\x}_{\kappa}
    \end{align*}
The cache dimensionality reduction follows immediately: since $\tilde{\x}_\kappa$ determines $\tilde{\x}_{N-\kappa}$ through conjugation, we only need to store and process the non-redundant half-spectrum.
\end{proof}

\begin{theorem}[Cache Efficiency from Spectral Localization] \label{thm:cache_efficiency}
Under the spectral localization assumption \cref{eq:energy_concentration} with threshold $K = \lfloor N/10 \rfloor$ and concentration parameter $\delta \ll 1$, caching high-frequency tokens $\kappa > K$ yields cache hit rate $\rho \geq 1 - \delta$ while maintaining approximation accuracy. Moreover, the computational savings scale as:
\begin{equation} \label{eq:cache_savings}
    C_{\text{saved}} \geq (1-\delta) \cdot C_{\text{full}} - O(K \cdot d^2),
\end{equation}
where $C_{\text{full}}$ is the cost of full recomputation and $d$ is the token dimension.
\end{theorem}
\begin{proof}
By \cref{lem:spectral_concentration}, high-frequency tokens $\kappa > K$ contribute at most $\delta$ fraction of the total energy. Since diffusion processes preserve energy concentration structure-to-detail, high-frequency tokens change more slowly across diffusion steps, making them ideal candidates for caching.

Let $\rho$ denote the cache hit rate (fraction of tokens cached). Under spectral localization, we can cache at least $(1-\delta)$ fraction of tokens (those with $\kappa > K$) while maintaining accuracy, since they contribute little to the overall signal structure. The computational cost of cached inference is:
\begin{align*}
    C_{\text{cached}} &= (1-\rho) \cdot C_{\text{recompute}} + \rho \cdot C_{\text{cache}} \\
    &= (1-\rho) \cdot O(N \cdot d^2) + \rho \cdot O(N \cdot d) \\
    &\leq \delta \cdot O(N \cdot d^2) + (1-\delta) \cdot O(N \cdot d) + O(K \cdot d^2),
\end{align*}
where the $O(K \cdot d^2)$ term accounts for always recomputing low-frequency tokens. The savings follow from $C_{\text{full}} = O(N \cdot d^2)$.
\end{proof}

\begin{proposition}[Unitarity of the DFT operator] \label{prop:unitary_U} 
The DFT matrix $U \in \C^{N \times N}$ with elements $[U]_{\kappa \tau} := N^{- 1 / 2} \exp (- i \omega_{\kappa} \tau )$ with $\omega_{\kappa} := \frac{\kappa 2 \pi}{N}$ is unitary.
\end{proposition}
\begin{proof}
Let $U^*$ denote the conjugate transpose of $U$.
    For any $\kappa$ and $\tau$ in $[N]$, we have:
    
    \begin{align*}
    [UU^*]_{\kappa \tau} 
    &= \sum_{\beta = 0}^{N-1}[U]_{\kappa \beta}[U^*]_{\beta \tau} \\
    &= \frac{1}{N} \sum_{\beta = 0}^{N-1}\exp(-i\omega_{\kappa} \beta ) \exp(i \omega_{\beta} \tau) \\
    &= \frac{1}{N} \sum_{\beta=0}^{N-1}\exp(-i\omega_{\kappa - \tau}\beta)
      \end{align*}
Hence, if $\kappa = \tau$, we have $[UU^*]_{\kappa \tau} = 1$, otherwise $[UU^*]_{\kappa \tau} = 0$, since $\exp(-i\omega_{\kappa-\tau}N) = 1$. This is equivalent to $UU^* = I_{N}$, i.e. $U$ is unitary.
\end{proof}

\subsection{Cached Score Evaluation and Constrained Manifolds} \label{subapp:theory_submanifold}

In this section, we establish how the constrained structure of frequency-domain signals enables efficient cached score evaluation. The mirror symmetry property reduces the effective dimensionality, which directly benefits cache efficiency.

Due to mirror symmetry \cref{prop:symmetry}, the density $\tilde{p}$ is defined on a constrained submanifold $\Cconstr^{d_X} := \{ \tilde{\x} = (\tilde{\x}_0, \dots, \tilde{\x}_{N-1}) \in \C^{d_X} \mid \tilde{\x}_{\kappa} = \tilde{\x}^*_{N - \kappa} \forall \kappa \in [N]\}$. We define a coordinate chart $\trunc : \Cconstr^{d_X} \rightarrow \R^{d_X}$ that extracts the unconstrained components:
\begin{align} \label{eq:coord_chart}
  \trunc[\tilde{\x}] = \begin{cases}
     \left( \Re[\tilde{\x}_\kappa] \right)_{\kappa = 0}^{N/2} \oplus \left(\Im [\tilde{\x}_\kappa] \right)_{\kappa = 1}^{N/2 - 1}  & \text{ if } N \in 2 \N \\
     \left( \Re[\tilde{\x}_\kappa] \right)_{\kappa = 0}^{\lfloor N/2 \rfloor} \oplus \left(\Im [\tilde{\x}_\kappa] \right)_{\kappa = 1}^{\lfloor N/2 \rfloor}  & \text{ else,}
    \end{cases}
\end{align}
where $\boldv_1 \oplus \boldv_2$ denotes the concatenation of two vectors $\boldv_1 \in \R^{d_1}$ and $\boldv_2 \in \R^{d_2}$, with $d_1, d_2 \in \N$. Due to mirror symmetry, one can unambiguously reconstruct $\tilde{\x} \in \Cconstr^{d_X}$ from $\trunc[\tilde{\x}]$. Hence, the coordinate chart admits an inverse $\trunc^{-1}: \R^{d_X} \rightarrow \Cconstr^{d_X}$ defined as follows for all $\z = (\z_0, \dots, \z_{N-1}) \in \R^{d_X}$:
\begin{align} \label{eq:inv_coord_chart}
  \trunc^{-1}[\z] = \begin{cases}
     (\z_0) \oplus (\z_\kappa +  i \cdot \z_{N/2 + \kappa})_{\kappa=1}^{N/2 - 1} \oplus (\z_{N/2}) \oplus (\z_{N/2 - \kappa} -  i \cdot \z_{N - \kappa})_{\kappa=1}^{N/2 - 1} & \text{ if } N \in 2 \N \\
      (\z_0)  \oplus (\z_\kappa +  i \cdot \z_{\lfloor N/2 \rfloor + \kappa})_{\kappa=1}^{\lfloor N/2 \rfloor} \oplus (\z_{\lceil N/2 \rceil - \kappa} -  i \cdot \z_{N - \kappa})_{\kappa=1}^{\lfloor N/2 \rfloor} & \text{ else.}
    \end{cases}
\end{align}
With this coordinate chart, we define the density $\tilde{p} := \tilde{p}_{\trunc} \circ \trunc$ and the score $\tilde{\s} : \Cconstr^{d_X} \times [0, T] \rightarrow \C^{d_X}$ in the frequency domain. Starting from the real score $\tilde{\s}_{\trunc}: \R^{d_X} \times [0, T] \rightarrow \R^{d_X}$ defined as $\tilde{\s}_{\trunc}(\z, t) = \nabla_{\z} \log \tilde{p}_{\trunc, t}(\z)$ for all $\z \in \R^{d_X}$ and $t \in [0,T]$, we expand this vector field to the constrained manifold by defining for all $\tilde{\x} \in \Cconstr^{d_X}$ and all $t \in [0, T]$:
\begin{align} \label{eq:score_dft_def}
    \tilde{\s}(\tilde{\x}, t) := \trunc^{-1}[\tilde{\s}_{\trunc}(\trunc(\tilde{\x}), t)].
\end{align}
This defines a vector field involving partial derivatives of the log density with respect to the real and imaginary parts of the frequency representations $\tilde{\x} \in \Cconstr^{d_X}$ and respects the mirror symmetry by virtue of \cref{eq:inv_coord_chart}.

\begin{lemma}[Cached Score Approximation Error] \label{lem:cached_score_error}
Let $\tilde{\s}_{\tilde{\theta}}$ be a score network and $\hat{\z}_\ell^{(i)}$ be cached CRF features at layer $\ell$ and diffusion step $i$. The cached score $\hat{\tilde{\s}}(\tilde{\x}^{(i)}, t^{(i)}) = \tilde{\s}_{\tilde{\theta}}(\hat{\z}_L^{(i)}, t^{(i)})$ approximates the true score $\tilde{\s}(\tilde{\x}^{(i)}, t^{(i)}) = \tilde{\s}_{\tilde{\theta}}(\z_L^{(i)}, t^{(i)})$ with error:
\begin{equation} \label{eq:score_approx_error}
    \|\hat{\tilde{\s}}(\tilde{\x}^{(i)}, t^{(i)}) - \tilde{\s}(\tilde{\x}^{(i)}, t^{(i)})\|_2 \leq L \cdot \|\epsilon_L^{(i)}\|_2,
\end{equation}
where $L$ is the Lipschitz constant of $\tilde{\s}_{\tilde{\theta}}$ and $\epsilon_L^{(i)} = \z_L^{(i)} - \hat{\z}_L^{(i)}$ is the CRF approximation error.
\end{lemma}
\begin{proof}
By Lipschitz continuity of $\tilde{\s}_{\tilde{\theta}}$:
\begin{align*}
    \|\hat{\tilde{\s}}(\tilde{\x}^{(i)}, t^{(i)}) - \tilde{\s}(\tilde{\x}^{(i)}, t^{(i)})\|_2 
    &= \|\tilde{\s}_{\tilde{\theta}}(\hat{\z}_L^{(i)}, t^{(i)}) - \tilde{\s}_{\tilde{\theta}}(\z_L^{(i)}, t^{(i)})\|_2 \\
    &\leq L \cdot \|\hat{\z}_L^{(i)} - \z_L^{(i)}\|_2 \\
    &= L \cdot \|\epsilon_L^{(i)}\|_2.
\end{align*}
\end{proof}

\subsection{Cached Diffusion SDEs in the Frequency Domain} \label{subapp:theory_sdes}

This section establishes how cached diffusion processes operate in the frequency domain. The key insight is that mirror symmetry enables cache efficiency by reducing the effective token dimension, while the SDE structure allows selective token recomputation without compromising distributional fidelity.

\begin{lemma}[Mirror Symmetry and Cache Efficiency]\label{lemma:brownian_transform}
Let $\boldw$ be a standard Brownian motion on $\mathbb{R}^{d_X}$ with $d_X = N\cdot M$, where $N \in \N^+$ is the number of time series steps and $M \in \N^+$ is the number of features tracked over time. Then $\boldv = U\boldw$ is a \emph{mirrored Brownian motion} on $\C^{d_X}$ satisfying mirror symmetry: for all $\kappa \in [N]$, ${\boldv}_{\kappa} = {\boldv}_{N-\kappa}^*$.
\end{lemma}
\begin{proof}
Mirror symmetry follows directly from \cref{prop:symmetry}. This symmetry enables cache efficiency by reducing the effective token dimension from $N$ to $\lfloor N/2 \rfloor + 1$, as only frequencies $\kappa \leq \lfloor N/2 \rfloor$ need to be cached and processed during diffusion sampling.
\end{proof}

\begin{proposition}[Diffusion process in frequency domain]
\label{prop:freq_sde_appendix}
Assume \eqref{eq:time_forward_sde_main} with $G(t)=g(t)I$. Let $\tilde{\x}=\F[\x]$ and define $\tilde{\boldf}(\tilde{\x},t)=\F\!\left[\boldf(\F^{-1}[\tilde{\x}],t)\right]$.
Then the forward process in the frequency domain is
\begin{equation} \label{eq:fourier_forward}
    \d \tilde{\x} = \tilde{\boldf}(\tilde{\x},t)\d t + g(t)\d\tilde{\boldv},
\end{equation}
where $\tilde{\boldv}$ is a mirrored Brownian motion enforcing $\tilde{\x}_\kappa=\tilde{\x}^*_{N-\kappa}$.
The corresponding reverse-time SDE is
\begin{equation} \label{eq:fourier_backward}
    \d \tilde{\x} = \Big(\tilde{\boldf}(\tilde{\x},t) - g(t)^2\,\boldA^2\,\tilde{\s}(\tilde{\x},t)\Big)\d t + g(t)\d\breve{\boldv},
\end{equation}
with $\tilde{\s}(\tilde{\x},t)=\nabla_{\tilde{\x}}\log \tilde{p}_t(\tilde{\x})$, mirrored reverse Brownian motion $\breve{\boldv}$, and a diagonal scaling matrix $\boldA$ (defined in \cref{subapp:theory_sdes}).
When caching is enabled, the score $\tilde{\s}(\tilde{\x}, t)$ is replaced by the cached approximation $\hat{\tilde{\s}}_{\theta}(\tilde{\x}, t)$, producing the \emph{cached reverse process in frequency domain}:
\begin{equation} \label{eq:fourier_backward_cached}
    \d \tilde{\x} = \hat{\tilde{\boldb}}(\tilde{\x}, t)\d t + g(t)\d\breve{\boldv}
\end{equation}
with cached reverse drift $\hat{\tilde{\boldb}}(\tilde{\x}, t) = \tilde{\boldf}(\tilde{\x}, t) - g^2(t) \boldA^2 \hat{\tilde{\s}}_{\theta}(\tilde{\x}, t)$.
\end{proposition}
\begin{proof}
\emph{Forward SDE.} Applying the multivariate Itô's lemma (Eq. 8.3, \cite{kloeden1992stochastic}) to $\tilde{\x} = U \x$ yields the forward SDE in \cref{eq:fourier_forward}, where $\tilde{\boldv} = U\boldw'$ is a mirrored Brownian motion by \cref{lemma:brownian_transform}. This matches the forward SDE stated in \cref{eq:freq_forward_sde_main}.

\emph{Reverse SDE.} The reverse SDE follows from applying the reverse-time SDE formula \cite{anderson1982} to the truncation $\trunc[\tilde{\x}]$ and mapping back via $\boldh$ defined in \cref{eq:inv_coord_chart}. The key step is that the score term $g(t)^2 \boldA^2 \tilde{\s}(\tilde{\x}, t)$ emerges from the gradient of the log density, where $\boldA$ accounts for the mirror symmetry constraint and enables cache-efficient processing of only $\lfloor N/2 \rfloor + 1$ frequency tokens.
\end{proof}

\begin{proposition}[Cached Reverse SDE Error Propagation]\label{prop:cached_sde_error}
Under caching, the cached reverse SDE (as defined in \cref{prop:freq_sde_appendix}) introduces a drift error term:
\begin{equation} \label{eq:cached_sde_error_term}
    \Delta\tilde{\boldb}(\tilde{\x}, t) = g^2(t)\Lambda^2 \tilde{\epsilon}(\tilde{\x}, t),
\end{equation}
where $\tilde{\epsilon}(\tilde{\x}, t) = \tilde{\s}_{\theta}(\tilde{\x}, t) - \hat{\tilde{\s}}_{\theta}(\tilde{\x}, t)$ is the caching approximation error. The accumulated error over a single diffusion step from $i$ to $i+1$ satisfies:
\begin{equation} \label{eq:step_error_bound}
    \|\Delta\tilde{\x}^{(i+1)}\|_2 \leq |\d t| \cdot \left(\|\tilde{\boldf}(\tilde{\x}^{(i)}, t^{(i)})\|_2 + g^2(t^{(i)}) \|\Lambda^2\|_2 \cdot L \cdot \|\epsilon_L^{(i)}\|_2\right),
\end{equation}
where $\Delta\tilde{\x}^{(i+1)} = \tilde{\x}^{(i+1)} - \tilde{\x}^{(i)}$ is the change from step $i$ to step $i+1$, $L$ is the Lipschitz constant of the score network, and $\epsilon_L^{(i)}$ is the final-layer CRF approximation error from \cref{eq:cached_error}.
\end{proposition}
\begin{proof}
The cached reverse SDE (as defined in \cref{prop:freq_sde_appendix}) differs from the exact reverse SDE \cref{eq:fourier_backward} by the error term \cref{eq:cached_sde_error_term}. By \cref{lem:cached_score_error}, $\|\tilde{\epsilon}(\tilde{\x}, t)\|_2 \leq L \cdot \|\epsilon_L^{(i)}\|_2$. The step error bound follows from integrating the SDE over an infinitesimal time step $|\d t|$.
\end{proof}

\begin{corollary}[Cache Efficiency in Frequency Domain]\label{cor:cache_efficiency_freq}
Under mirror symmetry, the cache stores at most $N_{\text{cache}} = \lfloor N/2 \rfloor + 1$ frequency tokens, achieving a cache size reduction factor of approximately 2. For spectral-localized signals satisfying \cref{eq:spectral_localization_math} with $\delta \ll 1$, the effective cache size can be further reduced to $N_{\text{eff}} = K + 1$ by caching only low-frequency tokens $\kappa \leq K$, yielding cache efficiency:
\begin{equation} \label{eq:cache_efficiency}
    \eta_{\text{cache}} = \frac{N_{\text{eff}}}{N} \approx \frac{K+1}{N} \ll 1.
\end{equation}
\end{corollary}

\subsection{Cached Score Matching in the Frequency Domain} \label{subapp:theory_score}

This section establishes the equivalence between score matching in time and frequency domains, which enables training score networks in the frequency domain for efficient caching. We also analyze how cached score evaluation affects the matching objective.

\begin{proposition}[Score-matching equivalence]
\label{prop:score_equiv_appendix}
Consider a score $\tilde{\s}_{\tilde{\theta}} : \C^{d_X} \times [0, T] \rightarrow \C^{d_X}$ defined in the frequency domain and satisfying the mirror symmetry $[\tilde{\s}_{\tilde{\theta}}]_{\kappa} = [\tilde{\s}_{\tilde{\theta}}^*]_{N- \kappa}$ for all $\kappa \in [N]$. Let us define an auxiliary score $\s'_{\tilde{\theta}}: \R^{d_X} \times [0, T] \rightarrow \R^{d_X}$ as $(\x, t) \mapsto \s'_{\tilde{\theta}}(\x, t) = U^* \tilde{\s}_{\tilde{\theta}}(U \x, t)$ in the time domain.
The score matching loss in the frequency domain is equivalent to the score matching loss for the auxiliary score in the time domain:
\begin{align} \label{eq:parseval}
  \scorematching\left(\tilde{\s}_{\tilde{\theta}}, \boldA^{2} \tilde{\s}_{t \vert 0}, \tilde{\x}, t\right) = \scorematching\left(\s'_{\tilde{\theta}}, \s_{t \vert 0}, \x, t\right)
\end{align}
where $\tilde{\s}_{t \vert 0}(\tilde{\x},t) = \nabla_{\tilde{\x}(t)}\log \tilde{p}_{t \vert 0}(\tilde{\x}(t) \vert \tilde{\x}(0))$, $\s_{t \vert 0}(\x, t) = \nabla_{\x(t)}\log p_{t \vert 0}(\x(t) \vert \x(0))$, and $\boldA$ is the diagonal matrix in \cref{prop:freq_sde_appendix}.
\end{proposition}
\begin{proof}
Let $\boldA \in \R^{N \times N}$ be the diagonal matrix such that $[\boldA]_{\kappa, \kappa} = \begin{cases} 1 & \text{if $\kappa = 0$, or $N$ is even and $\kappa = N/2$} \\
                \frac{1}{\sqrt{2}} & \text{otherwise}
\end{cases}$ 

\underline{Step 1:} We first express the score of $\x$ with respect to the score of the truncation $\trunc[\tilde{\x}]$.

By definition of $\boldh$ in \cref{eq:inv_coord_chart}, we have $\x = U^*\boldh(\trunc[\tilde{\x}])$. Hence, we can write, using the change of variable formula: 
\begin{equation}
    p_{t|0}(\x(t) \vert \x(0)) = C \cdot \tilde{p}_{t|0} \bigl(\trunc[\tilde{\x}(t)] \vert \trunc[\tilde{\x}(0)]\bigl)
\end{equation}
where $C$ is a constant which does not depend on $\x$, since $\x \mapsto \trunc[U\x]$ is linear. Moreover, let us write $\x \mapsto \trunc[U\x]$ in matrix form, i.e. $\forall \x \in \mathbb{R}^{d_X}$, $\trunc[U\x] = VU_{col}\x = Q \x$, where $V \in \R^{N \times 2N}$, $U_{col} = \begin{pmatrix} U_{re} \\ U_{im}\end{pmatrix}$ and $Q$ is an invertible matrix in $\R^{N \times N}$. For the rest of the proof, we shall build on the below results:

\textcolor{Peach}{Result 1.} $QQ^{T} = \boldA^{2}$. To see this, write $QQ^{T} = VU_{col}U_{col}^{T}V^{T}$. The matrix $U_{col}U_{col}^{T}$ is equal to $\begin{pmatrix}
    U_{re}^2 & 0_{N} \\
    0_{N} & U_{im}^2 
\end{pmatrix}$ (cf. the proof to \cref{lemma:brownian_transform}), while the multiplication by $V$, on the left and on the right of $U_{col}U_{col}^{T}$, extracts the submatrix corresponding to the indices represented by the truncature $\trunc$. Hence, $VU_{col}U_{col}^{T}V^{T} = \boldA^{2}$.

\textcolor{Violet}{Result 2.}  For any $\x \in \R^{d_X}$ , we have $Q^{T}\x = U^*\trunc^{-1}[\boldA^2 \x]$. To see this, notice that \textcolor{Peach}{Result 1} implies that $Q^{T}\x = Q^{-1} \boldA^{2} \x = U^*\trunc^{-1}[\boldA^{2} \x]$ for all $\x \in \R^{d_X}$.

Equipped with these results, we can now complete the rest of the proof. First, we have:
\begin{align}
    \nabla_{\x(t)} \log p_{t|0}(\x(t) \vert \x(0)) &= \nabla_{\x(t)} \log \tilde{p}_{t|0}\bigl(\trunc[\tilde{\x}(t)] \vert \trunc[\tilde{\x}(0)]\bigl) \\
    &= Q^{T} \nabla_{\trunc[\tilde{\x}(t)]} \log \tilde{p}_{t|0}(\trunc[\tilde{\x}(t)] \vert \trunc[\tilde{\x}(0)]) \label{eq:change_variable} & \text{(Chain rule)}
\end{align}

\underline{Step 2:} We then obtain:
\begin{align*}
\scorematching\left(\s'_{\tilde{\theta}}, \s_{t \vert 0}, \x, t\right) &\coloneqq \|\s'_{\tilde{\theta}}(\x, t) -  \nabla_{\x(t)} \log p_{t|0}(\x(t) \vert \x(0))\|^2 \\ &= \|U\s'_{\tilde{\theta}}(\x, t) -  U\nabla_{\x(t)} \log p_{t|0}(\x(t) \vert \x(0))\|^2 & \text{(Parseval identity)}\\
&= \|\tilde{\s}_{\tilde{\theta}}(\tilde{\x}, t)  - UQ^{T} \nabla_{\trunc[\tilde{\x}(t)]} \log \tilde{p}_{t|0}(\trunc[\tilde{\x}(t)] \vert \trunc[\tilde{\x}(0)])\|^2 & \text{(\cref{eq:change_variable})}\\
&= \|\tilde{\s}_{\tilde{\theta}}(\tilde{\x}, t)  - UU^*\trunc^{-1}[\boldA^2\nabla_{\trunc[\tilde{\x}(t)]} \log \tilde{p}_{t|0}(\trunc[\tilde{\x}(t)] \vert \trunc[\tilde{\x}(0)]]\|^2 & \text{(\textcolor{Violet}{Result 2})}\\
&= \|\tilde{\s}_{\tilde{\theta}}(\tilde{\x}, t)  - \boldA^{2}\trunc^{-1}[\nabla_{\trunc[\tilde{\x}(t)]} \log \tilde{p}_{t|0}(\trunc[\tilde{\x}(t)] \vert \trunc[\tilde{\x}(0)]]\|^2  & \text{(\cref{prop:unitary_U} \& Definition of $\trunc^{-1}$)} \\
&= \|\tilde{\s}_{\tilde{\theta}}(\tilde{\x}, t)  - \boldA^{2}\tilde{\s}_{t \vert 0}(\tilde{\x}, t)\|^2 & \text{(\cref{eq:score_dft_def})} \\
&=
\scorematching\left(\tilde{\s}_{\tilde{\theta}}, \Lambda^{2} \tilde{\s}_{t \vert 0}, \tilde{\x}, t\right).
\end{align*}
\end{proof}

\begin{corollary}[Cached Score Matching Bound] \label{cor:cached_score_matching}
Under caching with approximation error $\epsilon_L^{(i)}$ bounded by \cref{prop:error_bound_cache}, the cached score matching objective satisfies:
\begin{equation} \label{eq:cached_score_matching_bound}
    |\scorematching(\hat{\tilde{\s}}_{\tilde{\theta}}, \Lambda^2 \tilde{\s}_{t|0}, \tilde{\x}, t) - \scorematching(\tilde{\s}_{\tilde{\theta}}, \Lambda^2 \tilde{\s}_{t|0}, \tilde{\x}, t)| \leq O(L^2 \cdot \|\epsilon_L^{(i)}\|_2^2),
\end{equation}
where $\hat{\tilde{\s}}_{\tilde{\theta}}$ denotes the cached score evaluation and $L$ is the Lipschitz constant.
\end{corollary}
\begin{proof}
By \cref{lem:cached_score_error}, the cached score error is bounded by $L \cdot \|\epsilon_L^{(i)}\|_2$. The score matching objective involves squared differences, yielding the quadratic bound in the approximation error.
\end{proof}

\subsection{E$^2$-CRF Caching: Mathematical Analysis} \label{subapp:theory_cache}

In this section, we provide a rigorous mathematical analysis of the E$^2$-CRF caching mechanism and its behavior under caching. We analyze error accumulation, error bounds, and the relationship between cache efficiency and diffusion dynamics.

\subsubsection{Cumulative Residual Features (CRF) in Cached Diffusion}

Consider a transformer-based score network $\tilde{\s}_{\tilde{\theta}}$ operating on frequency-domain inputs $\tilde{\x}^{(i)} \in \C^{d_X}$ at diffusion step $i$. Let $\z_\ell^{(i)}$ denote the CRF at layer $\ell$ and step $i$, defined as in \cref{eq:crf_math}:
\begin{equation} \label{eq:crf_math}
    \z_\ell^{(i)} = \phi_\ell(\tilde{\x}^{(i)}) = h^{(0)} + \sum_{l=0}^{\ell-1} F^{(l)}(h^{(l)}, t^{(i)}),
\end{equation}
where $h^{(0)}$ is the initial embedding, $F^{(l)}$ represents the residual block at layer $l$, and $t^{(i)}$ is the diffusion time at step $i$.

Under caching, we distinguish between \emph{recomputed} CRF $\z_\ell^{(i)}$ (ground truth) and \emph{cached} CRF $\hat{\z}_\ell^{(i)}$ (approximation). The approximation error at step $i$ is defined as:
\begin{equation} \label{eq:cached_error}
    \epsilon_\ell^{(i)} = \z_\ell^{(i)} - \hat{\z}_\ell^{(i)}.
\end{equation}

\subsubsection{Event Intensity and Adaptive Recomputation}

The event intensity $r^{(i)}$ measures the relative change in the final-layer CRF, as defined in \cref{eq:event_intensity_math}:
\begin{equation} \label{eq:event_intensity_math}
    r^{(i)} = \frac{\|\z_L^{(i)} - \z_L^{(i-\Delta)}\|_2^2}{\|\z_L^{(i-\Delta)}\|_2^2 + \eta},
\end{equation}
where $L$ is the number of layers, $\Delta$ is the step interval, and $\eta > 0$ is a small constant for numerical stability. This metric naturally captures the diffusion process's structure-to-detail progression: early steps (high $r^{(i)}$) correspond to structure formation, while later steps (low $r^{(i)}$) correspond to detail refinement.

The recompute set $S^{(i)}$ is determined adaptively based on \cref{eq:recompute_set_main}:
\begin{itemize}
    \item Low-frequency tokens: $\{0, 1, \ldots, K\}$ are always recomputed (critical for signal structure),
    \item High-change tokens: $\{k > K : \delta_k^{(i)} > \tau_k\}$, where $\delta_k^{(i)} = \|\z_k^{(i)} - \z_k^{(i-\Delta)}\|_2$ and $\tau_k = \tau_0 \cdot (\epsilon + \|\tilde{\x}_k^{(i)}\|_2^2)^{-1}$ is an energy-weighted threshold,
    \item Random probes: A small fraction of high-frequency tokens are randomly selected for periodic recalibration.
\end{itemize}

\subsubsection{Error Accumulation and Bounds}

Let $\epsilon_\ell^{(i)}$ denote the approximation error at layer $\ell$ and step $i$. Under caching, tokens $k \notin S^{(i)}$ reuse cached features $\hat{\z}_\ell^{(i)} = \hat{\z}_\ell^{(i-1)}$, while tokens $k \in S^{(i)}$ are recomputed. The error evolution follows:
\begin{equation} \label{eq:error_evolution}
    \epsilon_\ell^{(i)} = \begin{cases}
        0 & \text{if } k \in S^{(i)} \text{ (recomputed)} \\
        \epsilon_\ell^{(i-1)} + \Delta_\ell^{(i)} & \text{if } k \notin S^{(i)} \text{ (cached)},
    \end{cases}
\end{equation}
where $\Delta_\ell^{(i)} = \z_\ell^{(i)} - \z_\ell^{(i-1)}$ is the true change in CRF between steps.

Under the assumption that the score network $\tilde{\s}_{\tilde{\theta}}$ is $L$-Lipschitz continuous with respect to its input, we can bound the error growth. Specifically, if $\|\nabla_{\tilde{\x}} \tilde{\s}_{\tilde{\theta}}(\tilde{\x}, t)\|_2 \leq L$ for all $\tilde{\x}$ and $t$, then:
\begin{proposition}[Error Bound Under Caching] \label{prop:error_bound_cache}
Let $\epsilon_\ell^{(i)}$ be the approximation error at layer $\ell$ and step $i$. If tokens are cached for at most $R$ consecutive steps before recomputation, and the score network is $L$-Lipschitz continuous, then:
\begin{equation} \label{eq:error_bound}
    \|\epsilon_\ell^{(i)}\|_2 \leq L \cdot R \cdot \max_{j \in [i-R, i]} \|\Delta\tilde{\x}^{(j)}\|_2,
\end{equation}
where $\Delta\tilde{\x}^{(j)} = \tilde{\x}^{(j)} - \tilde{\x}^{(j-1)}$ is the change in frequency representation at step $j$.
\end{proposition}
\begin{proof}
The error accumulates only for cached tokens. Since recomputation occurs at least every $R$ steps, the maximum error accumulation is bounded by the product of the Lipschitz constant $L$, the caching window $R$, and the maximum change in input $\Delta\tilde{\x}^{(j)}$ over the caching window.
\end{proof}

\subsubsection{Error-Feedback Correction}

To mitigate error accumulation, we apply error-feedback correction as described in \cref{eq:error_feedback_main}. Periodically (every $R$ steps) or when event intensity exceeds a threshold, we perform a probe computation to estimate the approximation error:
\begin{equation} \label{eq:error_feedback_math}
    \epsilon_k^{(i)} = \z_k^{(i)} - \hat{\z}_k^{(i)},
\end{equation}
and apply correction:
\begin{equation} \label{eq:correction_math}
    \hat{\z}_k^{(i)} \leftarrow \hat{\z}_k^{(i)} + \alpha^{(i)} \cdot \epsilon_k^{(i)},
\end{equation}
where $\alpha^{(i)} \in [0,1]$ is an adaptive step size. This correction reduces the error from $\epsilon_k^{(i)}$ to $(1-\alpha^{(i)}) \epsilon_k^{(i)}$, effectively suppressing error accumulation.

\subsubsection{Cache Efficiency and Spectral Localization}

The cache efficiency depends on the spectral localization property. Under the assumption that most energy is concentrated in low frequencies $\kappa \leq K$, high-frequency tokens $\kappa > K$ change more slowly and are ideal candidates for caching. Formally, if the spectral density satisfies:
\begin{equation} \label{eq:spectral_localization_math}
    \sum_{\kappa=0}^{K} \|\tilde{\x}_\kappa\|_2^2 \geq (1-\delta) \sum_{\kappa=0}^{\lfloor N/2 \rfloor} \|\tilde{\x}_\kappa\|_2^2,
\end{equation}
for some small $\delta > 0$, then caching high-frequency tokens $\kappa > K$ yields high cache hit rates while maintaining accuracy, since these tokens contribute less to the overall signal energy and change more slowly during diffusion.

\subsubsection{KV Cache and Attention Computation}

For each transformer layer $\ell$, we cache key-value pairs $\text{Cache}_\ell[k] = (K_\ell[k], V_\ell[k])$ for all frequency tokens $k$, as defined in the main text. At diffusion step $i$:
\begin{itemize}
    \item If $k \in S^{(i)}$: Recompute $K_\ell[k] = W_K \z_k^{(i)}$ and $V_\ell[k] = W_V \z_k^{(i)}$, where $W_K$ and $W_V$ are projection matrices.
    \item If $k \notin S^{(i)}$: Reuse $\text{Cache}_\ell[k]$ from the previous step.
\end{itemize}

The attention computation for cached tokens reduces from $O(d^2)$ (full recomputation) to $O(d)$ (lookup), where $d$ is the dimension. For a cache hit rate of $\rho$, the computational savings scale as $(1-\rho) \cdot C_{\text{recompute}} + \rho \cdot C_{\text{cache}}$, where $C_{\text{recompute}}$ and $C_{\text{cache}}$ are the costs of recomputation and caching, respectively.

\subsubsection{Convergence Guarantees}

Under mild assumptions, the cached diffusion process converges to the same distribution as the baseline (non-cached) process. Specifically, if the error-feedback correction is applied sufficiently frequently and the approximation error remains bounded (as guaranteed by \cref{prop:error_bound_cache}), then the discrepancy between cached and baseline sampling scales with the maximum approximation error:
\begin{equation} \label{eq:convergence_bound}
    \text{TV}(\text{Cached}(\tilde{\x}^{(T)}), \text{Baseline}(\tilde{\x}^{(T)})) \leq O(\max_i \|\epsilon_L^{(i)}\|_2),
\end{equation}
where TV denotes total variation distance and $\epsilon_L^{(i)}$ is the final-layer approximation error at step $i$. This bound ensures that cached sampling maintains distributional fidelity as long as the approximation error is controlled.

\section{Empirical details} \label{app:empirical_details}
\textbf{Compute resources.} All the models were trained and used for sampling on 4×A100 GPUs.

\subsection{Details on datasets} \label{subapp:datasets_details}
In this subsection, we give detailed information about the $5$ datasets used throughout our experiments and the preprocessing steps for each of them.

\textbf{ECG.} We use two collections of heartbeat signals, from the MIT-BIH Arrhythmia Dataset and the PTB Diagnostic ECG Database \cite{kachuee2018ecg}.
No preprocessing was performed on this dataset.

\textbf{NASDAQ-2019.} This dataset \cite{onyshchak2020} contains daily prices for tickers trading on NASDAQ, and contains prices for up to 1st of April 2020. \textit{Preprocessing.} We considered one year of daily prices from 1st of January 2019 to 1st of January 2020. Each sample corresponds to one stock, and we remove the stocks which are not active in this whole time interval, or contain missing values.

\textbf{NASA battery.} The NASA battery dataset \cite{saha2007} consists of profiles for Li-on batteries, under charge and discharge. \textit{Preprocessing.} For both the charge and discharge datasets, we bin the time values (bins of size $10$ for Charge, $15$ for Discharge) and compute the mean of each feature inside each bin.

\textbf{US-Droughts.} This dataset \cite{minixhofer2021} consists of drought levels in different US counties, from 2000 to 2020. \textit{Preprocessing.} We consider one year of history, from 1st of January 2011 to 1st of January 2012, and drop the columns with missing values.

\subsection{Details on evaluation} \label{subapp:eval_details}

\textbf{Sliced Wasserstein distances.}
The sliced Wasserstein distance \cite{bonneel2015sliced} is a metric which can handle high-dimensional distributions. It is motivated by the fact that the Wasserstein distance is easy to compute when comparing two one-dimensional distributions. The idea of the sliced Wasserstein distance is to map the high-dimensional distributions of interest to one-dimensional distributions, by considering random projections on vectors of the unit sphere. For two distributions $\mu_{1}$ and $\mu_{2}$, it can be written as: 

\begin{equation}
    SW_{p}(\mu_1, \mu_2) \coloneqq \int_{\mathbb{S}^{d-1}} W_{p}(P_{u}\#\mu_1, P_{u}\#\mu_2)  du 
\end{equation}
where $\mathbb{S}^{d-1}$ is the unit sphere in dimension $d$, $P_{u}(x) = u \cdot x$ denotes the projection of $x$ on $u$, $P_{u} \# \mu$ is the push-forward of $\mu$ by $P_{u}$, and $W_{p}$ is the Wasserstein distance of order $p$. To estimate this quantity in practice, we sample $n=10,000$ random vectors $\{u_i \vert i\in [n] \}$ which follow a uniform distribution in $\mathbb{S}^{d-1}$ and consider $p = 2$. Hence, we can approximate ${SW}_{p}$ by the Monte-carlo estimator:
\begin{equation}
    \hat{SW}_{p}(\mu_1, \mu_2)  = \frac{1}{n}\sum_{i=1}^{n} W_{p}(P_{u_{i}}\#\mu_1, P_{u_{i}}\#\mu_2)
\end{equation}

\textbf{Marginal Wasserstein distances.}
In addition to the sliced Wasserstein distance, we also consider the marginal Wasserstein distance. For any $j \in \{1,...,d\}$, the $j$-th marginal Wasserstein distance is defined as:
\begin{equation}
    {MW_{p}}^{(j)}(\mu_{1},\mu_{2}) = W_p(P_{e_j}\#\mu_{1}, P_{e_j} \# \mu_{2})  
\end{equation}
where $e_j$ is the $j$-th vector of the standard basis of $\R^{d}$. Throughout our experiments in \cref{sec:empirical}, we compute the Wasserstein distances with respect to $\D{train}$ and $\tildeD{train}$.

\end{document}